\documentclass{article}

\usepackage{amsmath,amsfonts,bm}

\def\eqref#1{equation~\ref{#1}}

\def\1{\bm{1}}

\def\vtheta{{\bm{\theta}}}

\def\vb{{\bm{b}}}
\def\vc{{\bm{c}}}

\def\vg{{\bm{g}}}
\def\vh{{\bm{h}}}

\def\vu{{\bm{u}}}
\def\vv{{\bm{v}}}

\def\vx{{\bm{x}}}
\def\vy{{\bm{y}}}
\def\vz{{\bm{z}}}

\def\mG{{\bm{G}}}

\def\mP{{\bm{P}}}

\def\mU{{\bm{U}}}
\def\mV{{\bm{V}}}
\def\mW{{\bm{W}}}
\def\mX{{\bm{X}}}

\def\mZ{{\bm{Z}}}

\def\mPhi{{\bm{\Phi}}}

\DeclareMathAlphabet{\mathsfit}{\encodingdefault}{\sfdefault}{m}{sl}
\SetMathAlphabet{\mathsfit}{bold}{\encodingdefault}{\sfdefault}{bx}{n}

\newcommand{\E}{\mathbb{E}}

\newcommand{\R}{\mathbb{R}}

\DeclareMathOperator*{\argmin}{arg\,min}

\usepackage{hyperref}
\usepackage{url}
\usepackage[margin=1in]{geometry}
\usepackage{algorithm}
\usepackage{booktabs}

\def\valpha{{\bm{\alpha}}}

\usepackage{amsthm}
\usepackage{subcaption}
\usepackage{graphicx}
\usepackage{xcolor}
\usepackage[noend]{algpseudocode}
\usepackage{dsfont}

\newtheorem{corollary}{Corollary}

\newtheorem{theorem}{Theorem}
\newtheorem{lemma}{Lemma}

\newtheorem{assumption}{Assumption}

\numberwithin{equation}{section} 
\numberwithin{lemma}{section} 
\numberwithin{theorem}{section}
\numberwithin{definition}{section} 
\numberwithin{corollary}{section}

\usepackage{wrapfig}

\usepackage{amsfonts}
\def \R {\mathbb{R}}

\begin{document}

\title{Efficient Techniques for Data Reconstruction, with Finite-Width Recovery Guarantees}
\author{Edward Tansley\thanks{Mathematical Institute, Woodstock Road, University of Oxford, Oxford, UK, OX2 6GG. {\tt edward.tansley@maths.ox.ac.uk.} This author's work was supported by the Mathematics of Random Systems CDT.},\quad Roy Makhlouf\thanks{ICTEAM Institute, UCLouvain, Euler Building, Avenue Georges Lemaître, 4 - bte L4.05.01, Louvain-la-Neuve, B - 1348, Belgium. {\tt roy.makhlouf@uclouvain.be.} Roy Makhlouf is a FRIA grantee of the Fonds de la Recherche Scientifique - FNRS.},\quad Estelle Massart\thanks{ICTEAM Institute, UCLouvain, Euler Building, Avenue Georges Lemaître, 4 - bte L4.05.01, Louvain-la-Neuve, B - 1348, Belgium. {\tt estelle.massart@uclouvain.be.} This author's work is partly funded by the FRS-FNRS Research Project NTTN (grant number TW02223) and the Concerted Research Action (ARC)  ``Gravit-AI''} \quad and \quad Coralia Cartis\thanks{Mathematical Institute, Woodstock Road, University of Oxford, Oxford, UK, OX2 6GG. {\tt coralia.cartis@maths.ox.ac.uk.} This author's work was supported by the Hong Kong Innovation and Technology Commission
(InnoHK Project CIMDA) and by the EPSRC grant EP/Y028872/1, Mathematical Foundations of Intelligence: An “Erlangen Programme” for AI.}}

\date{May 7, 2026}

\maketitle

\begin{abstract}
    Data reconstruction attacks on trained neural networks aim to recover the data on which the network has been trained and pose a significant threat to privacy, especially if the training dataset contains sensitive information. Here, we propose a unified optimization formulation of the data reconstruction problem based on initial and trained parameter values, incorporating state-of-the-art proposals. We show that in the 
    random feature model, this formulation provably leads to training data reconstruction with high probability, provided the network width is sufficiently large; this unprecedented finite-width result uses 
    PAC-style bounds. Furthermore, when the data lies in a low-dimensional subspace, we show that the network width requirement for successful reconstruction can be relaxed, with bounds depending on the subspace dimension rather than the ambient dimension. For general neural network models and unknown data orientations, we propose an efficient reconstruction algorithm that approximates the low-dimensional data subspace through the change in the first-layer weights during training and uses only the last-layer weights for reconstruction, thus
    reducing the search space dimension and the required network width for high-quality reconstructions.
    Our numerical experiments on synthetic datasets and CIFAR-10 confirm that our subspace-aware reconstruction approach outperforms standard full-space techniques.
\end{abstract}

\section{Introduction}
The rapid deployment of deep neural networks in privacy-sensitive domains, such as healthcare or identity verification, has raised critical concerns regarding the confidentiality of training data. Data reconstruction attacks seek to reconstruct training samples from a trained model using only publicly available information, such as the model's weights or architecture. These attacks can pose a threat to privacy, particularly when the model is trained on personal data. For example, in facial recognition systems, an attacker may be able to reconstruct facial images from the original training set, potentially exposing individuals without their consent. Understanding when such reconstruction is theoretically possible remains an open question.

Various types of data reconstruction methods have been proposed, including gradient-based \cite{wang2023reconstructing} and neural network attacks, that train a surrogate neural network for the reconstruction of a portion of the dataset \cite{balle2022reconstructing}. By contrast, our work here addresses strategies that exploit implicit biases of model training to reconstruct the data solely from the trained parameters. A first step in this direction is \cite{haim_reconstructing_2022}, which used the fact that, in the case of  overparameterized homogeneous ReLU feed-forward neural networks for binary classification with the logistic loss, gradient flow converges to the solution of a set of KKT equations of an associated maximum-margin problem, see \cite{lyu2019gradient} and \cite{ji_directional_2020}. KKT equations imply that the trained parameters are a linear combination of the gradients of the model with respect to the parameters at the training data that are ``on the margin''; inverting these equations then allows for reconstructing these data points. The theoretical foundation for these attacks relies on overparameterization: KKT equations provide a set of $p$ equations (the number of the model parameters) in $n d$ unknowns (the number of data points to recover multiplied by their dimension), which makes data reconstruction theoretically tractable. The authors of \cite{haim_reconstructing_2022} formulate data reconstruction as an optimization problem minimizing the violation of the KKT conditions, and validate it numerically. Their approach was subsequently extended to general losses and the multi-class setting \cite{buzaglo2023deconstructing}.

Building on these findings, the works \cite{loo2023understanding} and \cite{iurada_law_2025} aim to provably reconstruct the entire dataset for the neural tangent kernel (NTK) regime \cite{jacot2018neural} and the random feature (RF) model \cite{rahimi_random_2007,rahimi2008weighted}, respectively. In \cite{loo2023understanding}, the authors show that KKT conditions derived in the NTK regime for classification with mean squared error (MSE) loss ensure that the model parameters are a linear combination of the gradients of the neural network with respect to the model parameters at the data points. The authors of \cite{loo2023understanding} exploit this finding in the design of a novel data reconstruction algorithm that penalizes the misalignment between the trained parameters and the subspace spanned by these gradients and provide convergence guarantees establishing recovery of the entire dataset in the infinite-width limit. For RFs, it was shown in \cite{iurada_law_2025} that approximate reconstruction of the training samples becomes theoretically possible when the number of model parameters $p$ scales with $nd$. Iurada et al. \cite{iurada_law_2025} developed a reconstruction algorithm that minimizes the projection of the trained final-layer parameters onto the subspace spanned by the random features corresponding to the data, but do not provide recovery guarantees for this reconstruction algorithm, either in the infinite-width limit or in the finite case. The first goal of this paper is to provide a unified framework for data reconstruction with finite-width recovery guarantees that bridges the gap between the RF model analysis in \cite{iurada_law_2025} and the NTK approach in \cite{loo2023understanding}. 

\begin{figure}[t]
    \centering
    \includegraphics[width=0.99\textwidth]{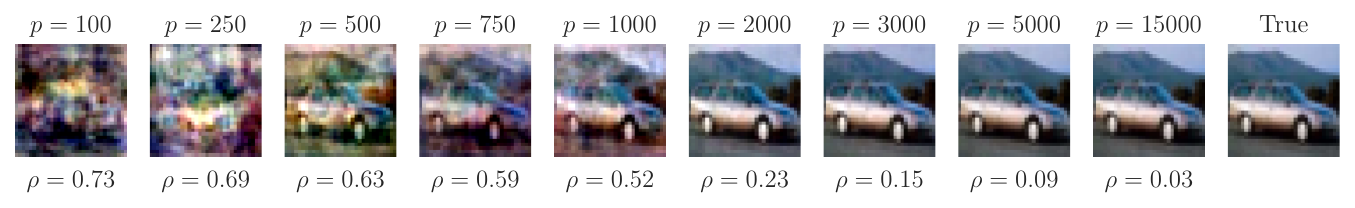}
    \caption{The quality of reconstruction improves as the network width $p$ increases, for $n=100$ data points. The reconstruction error $\rho$ displayed is averaged over the entire dataset.}
    \label{fig:1_image_reconstruction}
\end{figure}

As a second goal, we explore the potential benefits of structured data for the success of data reconstruction algorithms.
This is motivated by the common assumption that many datasets lie in a neighborhood of (possibly a union of) a low-dimensional manifold(s) \cite{rifai2011manifold, lee2007nonlinear}. This common belief is at the core of generative modeling strategies \cite{goodfellow2020generative} and was recently leveraged to understand the vulnerability of deep neural networks to adversarial attacks \cite{melamed2023adversarial, fawzi2018adversarialvulnerabilityclassifier}. We assume here the simplest form of such a structure, namely, the existence of a low-dimensional subspace in which the data lies, and explore the impact of this structure on the performance of our data reconstruction algorithm. We then propose a novel data reconstruction algorithm that exploits such low-dimensional structures.

Our contributions are as follows:
\begin{itemize}
    \item We establish PAC-style guarantees for data reconstruction in RF models, providing recovery guarantees for the algorithm proposed in \cite{iurada_law_2025}. Specifically, inspired by the methodology in \cite{loo2023understanding}, we show that, for RF models with sufficiently large width, the training data can be approximately recovered with high probability.

    \item We prove that when the training data lies in a low-dimensional subspace of dimension $r$, the theoretical width requirement depends on $r$ rather than the ambient dimension $d$.

    \item Exploiting such structure in the reconstruction process requires knowledge of the underlying subspace. In the case of 2-layer neural networks, we show numerically that this subspace is implicitly captured by the change in the first-layer weights during training. This leads to an efficient variant of the reconstruction algorithm in \cite{iurada_law_2025} that leverages this information to project the original search space onto a lower-dimensional one.

    \item We investigate the impact of network depth on the success of reconstruction. We show that reconstruction using only the last layer parameters still leads to strong results in deeper networks, while providing substantial computational benefits. In addition, when low-dimensional structure is present in the data, our subspace reconstruction methods outperform full space methods.
\end{itemize}

The paper is structured as follows. \autoref{sec:full_reconstr} introduces our unified formulation of reconstruction techniques for general data. In \autoref{sec:theory}, we present our theoretical recovery guarantees for RF models with finite width. The first subsection in 
\autoref{sec:subspace_reconstr} 
refines these guarantees when the data lies in a low-dimensional subspace, while \autoref{sec:alg-subspace}
contains our reconstruction algorithm that exploits the existence of this subspace for general net models. \autoref{sec:numerics} presents numerical results.

\section{Data reconstruction without data structure assumptions} \label{sec:full_reconstr}

Let $(\mX,\vy)$ be a dataset, where $\mX=[\vx_1,\ldots,\vx_n]^\top\in\R^{n\times d}$ contains $n$ training samples $\vx_1,\ldots,\vx_n\in\R^d$ and $\vy\in\R^n$ contains the corresponding labels. Throughout this work, we assume that a parameterized model $f(\cdot, \vtheta)$, with parameters $\vtheta\in\R^m$, is trained to minimize the MSE loss:
\begin{equation}\label{prob:train}  \tag{P-train}
    \min_{\vtheta \in \R^m} \mathcal{L}_{\mathrm{train}}(\vtheta; \mX, \vy),
\end{equation} 
where 
\[  \mathcal{L}_{\mathrm{train}}(\vtheta; \mX, \vy) := \frac{1}{2} \sum_{i=1}^n(f(\vx_i,\vtheta)-y_i)^2. \]

We assume that (\ref{prob:train}) is solved using gradient descent, starting from initialization $\vtheta_{0}$, converging to a point $\vtheta_{f}$, and we write $\Delta\vtheta:=\vtheta_{f}-\vtheta_0$. In this paper, we consider the task of reconstructing the data matrix $\mX$ solely from the knowledge of $\Delta\vtheta$ and the trained model gradients $\nabla_{\vtheta} f(\cdot,\vtheta_f)$. We let $\hat\mX = [\hat\vx_1,\ldots,\hat\vx_n]^\top \in\R^{n\times d}$ contain the $n$ reconstructed data $\hat\vx_1,\ldots,\hat\vx_n\in\R^d$. 

To ensure tractability of the reconstruction for homogeneous neural networks, we assume that the data is normalized as in \cite{loo2023understanding,iurada_law_2025}. 

\begin{assumption}\label{as:data_on_sphere}
    The training data $\vx_1,\ldots,\vx_n$ lie on the sphere $\sqrt{d}\mathcal{S}^{d-1}$.
\end{assumption}

For data reconstruction, we consider a loss that penalizes the misalignment between $\Delta \vtheta$ and the subspace spanned by the gradients of the model at the data points, i.e., we solve for $\hat \mX \in \R^{n \times d}$ and coefficients $\hat \valpha \in \R^n$
\begin{equation}\label{prob:univ_loss}  \tag{P-recon}
   \min_{\substack{\hat \mX \in \R^{n \times d}, \hat \valpha \in \R^n \\  \| \hat \vx_i \|_2 = \sqrt{d} \  \forall i\in[n]}} \mathcal{L}_{\mathrm{recon}}(\hat\mX ,\hat\valpha),
\end{equation}
where 
\[ \mathcal{L}_{\mathrm{recon}}(\hat\mX ,\hat\valpha) := \left\|\Delta \vtheta-\sum_{i=1}^n\hat\alpha_i\nabla_{\vtheta} f(\hat\vx_i, \vtheta_f)\right\|_2^2.\]
This reconstruction loss was initially proposed in \cite{haim_reconstructing_2022} for homogeneous ReLU neural networks trained with the logistic loss. It also matches the one used in \cite{loo2023understanding} (with $\nabla_{\vtheta} f(\hat\vx_i, \vtheta_f)$ replaced by $\nabla_{\vtheta} f(\hat\vx_i, \vtheta_0)$ due to the NTK assumption). We explore its use here for more general architectures. 

Note that (\ref{prob:univ_loss}) is equivalent to 
\begin{align*}
 \min_{\substack{\hat\mX  \in \R^{n \times d} \\ \| \hat \vx_i \|_2 = \sqrt{d} \  \forall i \in [n]}}  \mathcal{L}_{\mathrm{recon}}(\hat\mX,\hat\valpha^*(\hat \mX)),
\end{align*}  
where $\hat\valpha^*(\hat \mX) := \argmin_{\hat\valpha \in \R^n}  \mathcal{L}_{\mathrm{recon}}(\hat\mX,\hat\valpha)$ are the optimal coefficients given $\hat \mX$. Thus (\ref{prob:univ_loss})  can be formulated in terms of the projection of $\Delta \vtheta$ onto the orthogonal complement to the subspace spanned by the model gradients at the data points, as presented in \autoref{alg:reconstr_full_space}. 

\begin{algorithm}
\caption{Data reconstruction technique for general data structures}\label{alg:reconstr_full_space}
\begin{algorithmic}[1]
\State Let $\vtheta_f$ be obtained by applying gradient descent to problem (\ref{prob:train}), starting from initialization $\vtheta_0$, and let  $\Delta \vtheta := \vtheta_f-\vtheta_0$. 
\State Solve 
\begin{equation} \label{eq:proj_full_space}
\min_{\substack{\hat\mX = [\hat \vx_1, \dots, \hat \vx_n]^\top \in \R^{n \times d} \\ \| \hat \vx_i \|_2 = \sqrt{d} \  \forall i\in[n]}} \|\mP_{\mG}^\perp (\Delta \vtheta)\|_2^2,    
\end{equation}
where $\mP_{\mG}^\perp$ is the orthogonal projection onto the orthogonal complement of the subspace spanned by the rows of $\mG = [\nabla_{\vtheta} f(\hat\vx_1, \vtheta_f), \dots, \nabla_{\vtheta} f(\hat\vx_n, \vtheta_f)]^\top$.
\State Return the reconstructed data matrix $\hat \mX$.
\end{algorithmic}
\end{algorithm}

 While (\ref{prob:univ_loss}) and (\ref{eq:proj_full_space}) are equivalent, we rely on (\ref{prob:univ_loss}) in our theoretical results, which  explicitly involves the coefficients $\hat \valpha \in \R^n$, and use the more computationally friendly expression (\ref{eq:proj_full_space}) in our numerical experiments, the computation of which is detailed in \autoref{sec:further_numerics_details}.  Note that formulation (\ref{eq:proj_full_space}) also coincides with the reconstruction proposed in \cite{iurada_law_2025} for RF models. Then, $\vtheta$ only contains the second-layer parameters (as the first-layer ones are random) and the rows of $\mG$ are the random features of the model, see next section. When departing from RF models, the authors of \cite{iurada_law_2025} do not consider all model parameters but only the ones associated with the last layer; we will compare both approaches numerically in \autoref{sec:numerics}.

\section{Non-asymptotic guarantees for recovery in random features models} \label{sec:theory}

In this section, assuming that gradient descent applied to (\ref{prob:train}) has converged exactly to a minimizer $\vtheta^*$, we derive theoretical recovery guarantees for \autoref{alg:reconstr_full_space} applied to RF models in the case where (\ref{eq:proj_full_space}) is solved exactly (equivalently, for reconstructed data matrix $\hat \mX^*$ and coefficients $\hat\valpha^*$ solving exactly (\ref{prob:univ_loss})). Specifically, we consider a two-layer neural network in which only the second layer is trained, while the first layer is fixed and randomly initialized. In contrast with existing works, our results are non-asymptotic, i.e., do not assume the network width to tend to infinity. Note that recovery guarantees were already obtained in the NTK regime, in the infinite-width limit \cite{loo2023understanding}. 

We consider the random feature (RF) model 
\begin{equation}\label{eq:rf_model_with_bias}
    f_{\mathrm{RF}}(\vx,\vtheta)=\frac{1}{\sqrt{p}}\phi(\mV\vx+\vb)^\top\vtheta=\varphi(\vx)^\top\vtheta,
\end{equation} where $\mV_{i,j}\sim_{\text{i.i.d}}\mathcal{N}(0, 1/d)$, $\vb_i\sim_{\text{i.i.d}}\mathcal{N}(0, 1)$ and $\varphi(\vx):=\frac{1}{\sqrt{p}}\phi(\mV\vx+\vb)$, see \cite{rahimi_random_2007,rahimi2008weighted}. 
The corresponding kernel is given by $k_p(\vx, \vx')=\frac{1}{p}\phi(\mV\vx+\vb)^\top \phi(\mV\vx'+\vb)=\varphi(\vx)^\top\varphi(\vx')$, with infinite-width limit $k_{\infty}(\vx, \vx')=\E[\phi(\vv^\top \vx+b)\phi(\vv^\top \vx'+b)]$, where $\vv_i\sim_{\text{i.i.d}}\mathcal{N}(0, 1/d)$ and $b\sim_{\text{i.i.d}}\mathcal{N}(0, 1)$. We denote the related reproducing kernel Hilbert spaces (RKHSs) by $\mathcal{H}_p$ and $\mathcal{H_\infty}$, respectively. We let $\mPhi:= [\varphi(\vx_1),\ldots,\varphi(\vx_n)]^{\top}$ be the feature map corresponding to the training data. We further make the following assumption.

\begin{assumption}\label{as:activation_function}
    The activation function $\phi$ is continuous, nonlinear, not a polynomial, bounded and has a bounded derivative, i.e., $|\phi(x)|\leq M$ and $|\phi'(x)|\leq L$ for all $x\in\R$.
\end{assumption}

\autoref{as:activation_function} covers a wide family of activations, such as the hyperbolic tangent and the sigmoid. A nonlinear and nonpolynomial activation is required to ensure universality of the kernel $k_\infty$ \cite{sun_approximation_2019}, a property that will be used to derive recovery guarantees for \autoref{alg:reconstr_full_space}.

We let $\vtheta_0=0$ for simplicity. With any choice of step size that ensures convergence, gradient descent applied to problem (\ref{prob:train}) in the overparameterized regime ($p>n$) converges to the minimum-norm interpolator, i.e., $\vtheta_f = \vtheta^*$ with 
\begin{equation}\label{eq:theta_opt_value}
    \vtheta^*=\mPhi^+\vy=\mPhi^\top(\mPhi\mPhi^\top)^{-1}\vy=\mPhi^\top\valpha,
\end{equation}
where we assumed that $\mPhi\mPhi^\top$ is invertible, and with $\valpha:= (\mPhi \mPhi^\top)^{-1} \vy$, see \cite[eq. 3.3]{bartlett2021deep}.

As mentioned in \autoref{sec:full_reconstr}, since $\nabla_\vtheta f_\mathrm{RF}(\cdot, \vtheta)=\varphi(\cdot)$, setting the parameterized model in (\ref{prob:univ_loss}) to be $f\equiv f_{\mathrm{RF}}$ reduces \autoref{alg:reconstr_full_space} to the one proposed in \cite{iurada_law_2025}, for which no recovery guarantees have been established. In particular, note that $\mG = \mPhi$ in the RF setting. 

\subsection{Data reconstruction guarantees without data structure assumptions}
Inspired by the work \cite{loo2023understanding}, we provide recovery guarantees for \autoref{alg:reconstr_full_space} --equivalently, for $(\mX^*, \valpha^*)$ solving (\ref{prob:univ_loss}) exactly-- in the case of RF models with finite-width $p$, showing that the training data can be approximately reconstructed when $p$ is sufficiently large. In order to achieve this, we first need to quantify the deviation of the finite-width kernel from the infinite-width one, which is the focus of the next lemma.

\begin{lemma}\label{lem:approx_kernel_bound}
    Let \autoref{as:data_on_sphere} and \autoref{as:activation_function} hold. Let $\varepsilon>0$ and $\delta>0$ be such that \begin{equation}\label{eq:delta_condition}
        \delta\leq \frac{4(1+2d)LM}{\sqrt{d}\varepsilon}.
    \end{equation} If \begin{equation}\label{eq:p_condition}
        p\geq \frac{8M^4(2d+1)}{\varepsilon^2}\left(\ln\left(\frac{2+4d}{\delta}\right)+\frac{2d}{2d+1}\ln\left(\frac{6LM}{\sqrt{d}\varepsilon}\right)\right),
    \end{equation} then \begin{equation}\label{eq:proba_lower_bound_on_approx_kernel}
        \mathbb P\left[\sup_{\vx,\vx'\in \sqrt{d}\mathcal{S}^{d-1}}|k_p(\vx, \vx')-k_\infty(\vx, \vx')|\leq\varepsilon\right]\geq 1-\delta.
    \end{equation}
\end{lemma}

\begin{proof}[Proof sketch]  
    We follow the same methodology as in the proof of Claim 1 in \cite{rahimi_random_2007}, which addresses translation-invariant kernels. Computations are here adapted to account for the fact that the RF kernel is in general not translation-invariant. The complete proof is given in \autoref{ap:proof_theoretical_results}.
\end{proof}

From \autoref{lem:approx_kernel_bound}, we now derive the following theorem, whose proof is given in \autoref{ap:proof_theoretical_results}.

\begin{theorem}\label{thm:finite_recon_proof}
    Let \autoref{as:data_on_sphere} and \autoref{as:activation_function} hold, and assume that the training samples satisfy $\|\vx_i-\vx_j\|_2\geq\Delta$ for all $i\ne j$ and for some $\Delta >0$. Additionally, assume that the coefficients $\alpha_i$ in (\ref{eq:theta_opt_value}) are bounded away from zero, i.e., $|\alpha_i|\geq c>0$ for all $i\in[n]$. Let $\hat \mX^*$ and $\hat\valpha^*$ be optimal solutions to (\ref{prob:univ_loss}), for $\vtheta_f = \vtheta^*$ given in (\ref{eq:theta_opt_value}). If $\delta>0$ satisfies  $$\delta\leq16C(1+2d)LM\frac{(\|\valpha\|_1+\|\hat\valpha\|_1)^2}{\sqrt{d} c^2}$$ for some positive constant $C$ that depends on $\Delta$, and if $p$ satisfies $$p\geq \frac{128C^2M^4(2d+1)(\|\valpha\|_1+\|\hat\valpha\|_1)^4}{c^4}\left(\ln\left(\frac{2+4d}{\delta}\right)+\frac{2d}{2d+1}\ln\left(\frac{24CLM(\|\valpha\|_1+\|\hat\valpha\|_1)^2}{\sqrt{d}c^2}\right)\right),$$ then, with probability $1-\delta$, for all $j\in[n]$, there exists a reconstructed sample $\hat\vx_i$ such that $\|\vx_j-\hat\vx_i\|_2<\Delta$.
\end{theorem}

\begin{proof}[Proof sketch]
    This result follows from standard connections between RKHS norms and the maximum mean discrepancy (MMD) between measures \cite{gretton2012kernel}. For universal kernels defined on compact sets, the MMD is injective \cite{sriperumbudur2011universality}. Based on this property, the work \cite{loo2023understanding} showed that, in the infinite-width limit of the NTK regime, a vanishing MMD implies exact recovery of the training data. In our setting, the kernel $k_\infty$ remains universal \cite{sun_approximation_2019}, so injectivity of the MMD still holds. Although the MMD is no longer zero for a finite-width $p$, using its dual characterization \cite{vayer_controlling_2023} and \autoref{lem:approx_kernel_bound}, it can be made arbitrarily small as the width gets larger, which in turn yields approximate recovery of the training data.
\end{proof}

\autoref{thm:finite_recon_proof} proves that, for finite-width RF models, the training data can be approximately reconstructed from the trained weights and model architecture alone, using \autoref{alg:reconstr_full_space}.  

\section{Data reconstruction for data with special structure} \label{sec:subspace_reconstr}

 If we assume that the data has some additional structure, one can hope to reduce the network width requirement for data reconstruction to be possible. In this regard, we make the following assumption, as an alternative to \autoref{as:data_on_sphere}.

\begin{assumption}\label{assumption:low-rank_data}
    The rows of $\mX$ lie on the intersection of the sphere $\sqrt{d}\mathcal{S}^{d-1}$ and some $r$-dimensional subspace $\mathcal{X}\subseteq\R^d$ spanned by the columns of an orthogonal matrix $\mU\in\R^{d\times r}$, for some $r < d$.
\end{assumption}

Let $\mU$ be defined in \autoref{assumption:low-rank_data}. Then, note that (\ref{prob:univ_loss}) is equivalent to the lower-dimensional problem
\begin{equation} \label{prob:univ_loss_reduced}  \tag{P-recon-subspace}
    \min_{\substack{\hat \mZ \in \R^{n \times r}, \hat \valpha \in \R^n \\ \|\hat \vz_{i}\|_2 = \sqrt{d}\ \forall i\in[n]}}
    \mathcal{L}_{\text{recon}}(\hat \mZ \mU^\top, \hat \valpha ).
\end{equation}
In particular, if $(\hat \mZ^*, \hat \valpha^*)$ is an optimal solution of (\ref{prob:univ_loss_reduced}), then $(\hat \mZ^* \mU^\top, \hat \valpha^*)$ is an optimal solution of (\ref{prob:univ_loss}). 

\subsection{Recovery guarantees for RF models}
Under this data assumption, we may use the rotational invariance of the Gaussian distribution to derive the following corollary, for RF models described in last section. 

\begin{corollary}\label{cor:finite_recon_proof_subspace}
    Let \autoref{assumption:low-rank_data} hold. Let $(\hat \mZ^*, \hat \valpha^*)$ be any optimal solution of (\ref{prob:univ_loss_reduced}) for $\vtheta_f = \vtheta^*$ given in (\ref{eq:theta_opt_value}), and let $\hat \mX^* :=  \hat \mZ^* \mU^\top$. Then, $(\hat \mX^*, \hat \valpha^*)$ benefits from the recovery guarantees of \autoref{thm:finite_recon_proof} (under the same assumptions), replacing $d$ by $r$ in the bounds on $\delta$ and $p$.
\end{corollary}

\begin{proof}[Proof]
    For any $\vx\in\sqrt{d}\mathcal{S}^{d-1}\cap\mathcal{X}$, let $\vz \in\sqrt{d}\mathcal{S}^{r-1}$ be such that $\vx = \mU \vz$, where $\mU$ is defined in \autoref{assumption:low-rank_data}. By the rotational invariance of the Gaussian distribution, there holds $$\varphi(\vx)=\frac{1}{\sqrt{p}}\phi\left(\frac{\sqrt{d}}{\sqrt{r}}\mV\mU\frac{\sqrt{r}}{\sqrt{d}}\vz+\vb\right)=\frac{1}{\sqrt{p}}\phi(\mV' \vz'+\vb),$$ where $\mV'_{i,j}:=\frac{\sqrt{d}}{\sqrt{r}}(\mV\mU)_{i,j}\sim\mathcal{N}(0, 1/r)$ and $\vz':=\frac{\sqrt{r}}{\sqrt{d}}\vz \in\sqrt{r}\mathcal{S}^{r-1}$. Hence, we can apply the same argument as in the proof of \autoref{thm:finite_recon_proof}, with $d$ replaced by $r$ and $\Delta':=\frac{\sqrt{r}}{\sqrt{d}}\Delta$, to get that, with probability $1-\delta$, for each low-dimensional representative $\vz_j'\in\sqrt{r}\mathcal{S}^{r-1}$ of a data point $\vx_j$ (i.e., $\vz_j'$ satisfies $\vx_j = \frac{\sqrt{d}}{\sqrt{r}}\mU \vz_j'$), there exists an estimated low-dimensional data point $\hat{\vz}_i'\in\sqrt{r}\mathcal{S}^{r-1}$ such that $\|\vz_j'-\hat{\vz}_i'\|<\Delta'$. This implies that $\|\vx_j-\hat\vx_i\| = \| \frac{\sqrt{d}}{\sqrt{r}} \mU (\vz_j'-\hat{\vz}_i') \| <\Delta$, which concludes the proof.
\end{proof}

\subsection{A practical reconstruction algorithm for data with special structure}
\label{sec:alg-subspace}
While the previous section shows that the low-dimensional data reconstruction problem (\ref{prob:univ_loss_reduced}) allows data recovery under the assumption that the data subspace is known, this assumption is often too strong in practice. We propose here a strategy to learn the data subspace based on the first-layer parameters. As this section addresses algorithmic considerations, we move away from RF models and revert back to the general setting described in \autoref{sec:full_reconstr}.

We estimate the subspace to which the data belongs as the span of the leading right singular vectors of $\Delta \mW_{1} := \mW_{1,f} - \mW_{1,0}$, where $\mW_1$ are the first-layer weights of the neural network $f$ (we assume here the dimension of the subspace to be known\footnote{Note that, when it is not the case, a sharp decrease in the singular values of $\Delta \mW_{1}$ provides an estimation of the subspace dimension, see \autoref{fig:spectral_decay_subspace_tests} (left).}). Indeed, by the chain rule,
\begin{equation}
    \frac{\partial \mathcal{L}_{\text{train}}(\vtheta)}{\partial \mW_{1}} = \sum_{i = 1}^n (f(\vx_i, \vtheta)-y_i) \frac{\partial f(\vx_i, \vtheta)}{\partial \vh_{1,i}} \frac{\partial \vh_{1,i}}{\partial \mW_1} = \sum_{i = 1}^n (f(\vx_i, \vtheta)-y_i) \frac{\partial f(\vx_i, \vtheta)}{\partial \vh_{1,i}} \vx_i^\top,
\end{equation}
where $\vh_{1,i}$ is the first-hidden-layer preactivation of $f$ for input data $\vx_i$. It follows that without weight decay, the rows of any updates of $\mW_1$ lie in the span of the input data. Based on this observation, we propose a novel data reconstruction algorithm, \autoref{alg:reconstr_subspace}, for data satisfying \autoref{assumption:low-rank_data}.

\begin{algorithm}
\caption{Data reconstruction technique when the data lies in an $r$-dimensional subspace}\label{alg:reconstr_subspace}
\begin{algorithmic}[1]
\State Let $\vtheta_f$ be obtained by applying gradient descent to problem (\ref{prob:train}), starting from initialization $\vtheta_0$, and let  $\Delta \vtheta := \vtheta_f-\vtheta_0$. 
\State Compute $\hat \mU$, an approximation of $\mU$, as a basis of the span of the leading right singular vectors of $\Delta \mW_1 := \mW_{1,f}- \mW_{1,0}$. 
\State Solve 
\begin{equation} \label{eq:proj_subspace}
\min_{\substack{\hat\mZ = [\hat \vz_1, \dots, \hat\vz_n]^\top \in \R^{n \times r} \\ \| \vz_i \|_2 = \sqrt{d} \  \forall i\in[n]}} \|\mP_{\mG}^\perp (\Delta \vtheta)\|_2^2,    
\end{equation}
where $\mP_{\mG}^\perp$ is the orthogonal projection onto the orthogonal complement of the subspace spanned by the rows of $\mG = [\nabla_{\vtheta} f(\hat \mU^\top \hat\vz_1, \vtheta_f), \dots, \nabla_{\vtheta} f(\hat \mU^\top \hat\vz_n, \vtheta_f)]^\top$.
\State Return the reconstructed data matrix $\hat \mX := \hat \mZ \hat\mU^\top$.
\end{algorithmic}
\end{algorithm}

Using the first-layer parameters to learn the data subspace questions the need to then use all parameters in the reconstruction algorithm.
Indeed, the reconstruction method proposed in \cite{iurada_law_2025} was already only relying on the last-layer parameters (but without learning the data subspace). We will show numerically in next section that using the first-layer parameters for subspace learning and the last-layer for reconstruction are sufficient in practice to recover the dataset, leading to strong computational benefits.

\section{Numerical results}  \label{sec:numerics}

 In this section, we compare the different reconstruction algorithms discussed in the paper, for real and synthetic datasets and various architectures.

\paragraph{Data generation.}

Our experiments involve two types of data: CIFAR-10 images \cite{lecun_gradientbased_1998} and synthetic low-dimensional data.
For the CIFAR-10 data, we use $n = 10$ or $100$ images, evenly selected from each of $10$ classes (thus, corresponding to 1 or 10 images per class, respectively). As such, the dataset belongs to an $n$-dimensional linear subspace of $\R^{3 \times 32 \times 32}$, so that \autoref{alg:reconstr_subspace} can be applied with $r = n$. The low-dimensional data points are generated within a randomly-generated $r$-dimensional subspace, and normalized such that $\|\vx_i\|_2 = \sqrt{d}$ for all $i = 1, \dots, n$. The labels are given by $y_{i} = \vg^{\top}\vx_{i} + \epsilon_i$ for some random vector $\vg$ with $g_{j} \sim N(0, 1/d)$ and noise $\epsilon_i \sim N(0, \sigma^{2})$ for noise level $\sigma$. In our experiments, we set $n = 100, d = 60, r = 30,$ and $\sigma = 0.5$. 

\paragraph{Training and reconstruction algorithms considered.}

We consider $L$-layer feedforward neural networks with ReLU activations, which are bias-free for simplicity, parametrized by their weight matrices $\vtheta = (\mW_1,\ldots, \mW_L)$.  This architecture is trained on the dataset to be recovered using gradient descent with learning rate $\eta = 10^{-4}$ until the value of the MSE loss reaches $10^{-7}$.

\begin{figure}[b]
    \centering
\includegraphics[width=0.54\textwidth]      {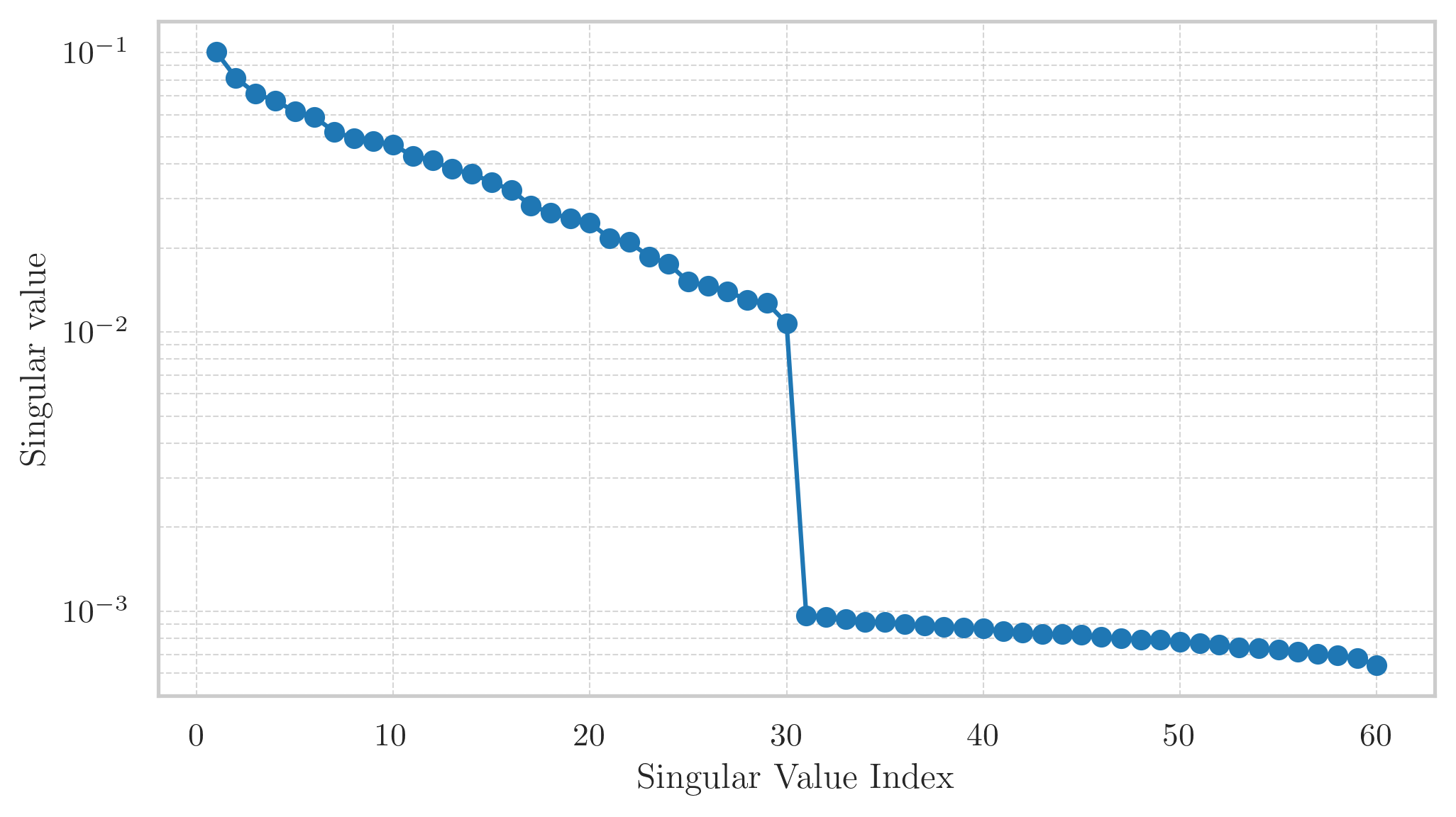}\hfill
\includegraphics[width=0.45\textwidth]{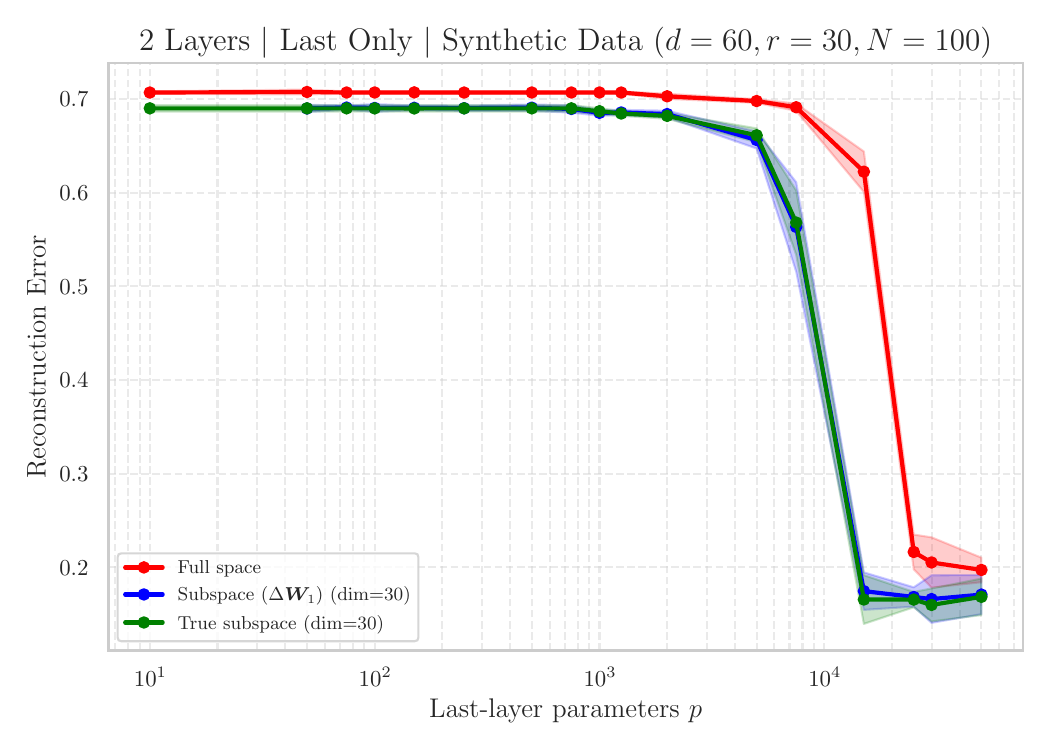}

    \caption{\textbf{Left}: The spectral decay of $\Delta\mW_1$ in a 2-layer network of width $p=10^3$, trained on data drawn from a 30-dimensional subspace of $\R^{60}$; \textbf{Right}: \autoref{alg:reconstr_full_space} (``Full space'') and \autoref{alg:reconstr_subspace} (``Subspace ($\Delta \mW_1$)'' and ``True subspace'') for synthetic data reconstruction using  last-layer parameters only (see Results section). }
    \label{fig:spectral_decay_subspace_tests}
\end{figure}

For data reconstruction, we consider two settings in terms of trained parameters that are used. Either data reconstruction is done using all network parameters or using only the last layer parameters, i.e., \autoref{alg:reconstr_full_space} and \autoref{alg:reconstr_subspace} are applied either to all the network parameters $\Delta \vtheta = (\mW_{1,f}-\mW_{1,0},\ldots,\mW_{L,f}-\mW_{L,0})$ or to 
a subset of parameters (last layer), $ \Delta\vtheta=\Delta\tilde{\vtheta} := (\mW_{L,f}-\mW_{L,0})$.

For each choice of parameters ($\Delta \vtheta$ and $\Delta\tilde \vtheta$), we compare three methods:
\begin{itemize}
    \item \textbf{Full space reconstruction:}  \autoref{alg:reconstr_full_space}  in which (\ref{eq:proj_full_space}) is solved by projected gradient descent with momentum $\beta = 0.9$, learning rate $20$ for synthetic data and $2\times 10^3$ for CIFAR-10, following  $\cite{iurada_law_2025}$. We run the reconstruction algorithm for $10^4$ iterations
    \item \textbf{Subspace ($\Delta \mW_1$) reconstruction:} \autoref{alg:reconstr_subspace} in which (\ref{eq:proj_subspace}) is solved using projected gradient descent. We use the same algorithm hyperparameters as the full space algorithm.
    \item \textbf{True subspace:} A variant of \autoref{alg:reconstr_subspace} in which $\mU$ is given (instead of being estimated). Note that $\mU$ is generally unknown in practice, so this method is here only for comparison.
\end{itemize}

In all experiments, we display the data reconstruction rate in terms of the size of the architecture. In agreement with \cite{iurada_law_2025}, the size of the model is here represented by its width, constant across layers. In the cases where the network output is multivariate, we count the number of last layer parameters $p^{(L)}$ rather than the network width $p$. In the CIFAR-10 experiments, we have $p^{(L)} = 10p$. We run each experiment with 3 random seeds for the training algorithm and 5 for the reconstruction algorithm. We include the standard deviation of these 15 runs in the plots in addition to the mean error.

\paragraph{Error measure.} Similarly to \cite{iurada_law_2025}, we measure the performance of the data reconstruction algorithms using the mean $L^2$ error between each reconstructed image and its closest image from the dataset:
\begin{equation}
     \rho(\mX, \hat{\mX}) = \min_{\Pi \in \mathcal{P}_n} \frac{1}{n \sqrt{d}} \sum_{i=1}^n \| \vx_i - \hat \vx_{\Pi(i)} \|_2,
\end{equation}
where $\mathcal{P}_n$ is the set of permutations of the $n$ data points. Note that, for a fair comparison, we rescale the true data to have $\sqrt{d}$ norm when computing the error (as \autoref{alg:reconstr_full_space} and \autoref{alg:reconstr_subspace} output images with norm $\sqrt{d}$). 

Further experiment details can be found in \autoref{sec:further_numerics_details}.

\paragraph{Results.}

In \autoref{fig:spectral_decay_subspace_tests} (right), we compare \autoref{alg:reconstr_full_space} and \autoref{alg:reconstr_subspace} for data reconstruction using the last-layer parameters only, on our synthetic dataset. We see that $\Delta\mW_1$ captures the 30-dimensional data subspace (see \autoref{fig:spectral_decay_subspace_tests} (left)), and  \autoref{alg:reconstr_subspace} (that learns the data subspace using $\Delta\mW_1$) performs as well as the ``true subspace'' method; furthermore, subspace methods outperform the full space reconstruction method, facilitating data reconstruction with approximately half the network width compared to the full method.
In \autoref{fig:depth_effect_last_layer}, we run these same methods on the CIFAR-10 dataset, using a 2-layer and a 5-layer network. This indicates that data reconstruction is still reasonably achieved for the 5-layer networks.

Finally, we compare the performance of the reconstruction methods (\autoref{alg:reconstr_full_space} and \autoref{alg:reconstr_subspace}) relying only on the last-layer parameters ($\Delta\tilde{\vtheta}$) with those using all trainable parameters ($\Delta\vtheta$) in \autoref{fig:first_layer_vs_all_layers} on the CIFAR-10 dataset, for networks with various depths. While for moderate widths using all parameters seems beneficial, for larger widths, using only last layer parameters is superior, and also considerably faster than the methods using all parameters. Additional experiments are presented in \autoref{sec:additional_plots}. 

\paragraph{Conclusion.}
We introduced a unified formulation of the data reconstruction problem based on initial and trained parameter values, and derived finite-width recovery guarantees for data reconstruction using \autoref{alg:reconstr_full_space} for RF models.
We further demonstrated that the network width requirement for successful reconstruction is related to the intrinsic dimensionality of the data. Specifically, when the training data lies in a low-dimensional subspace, this requirement can be relaxed.
To account for situations where the data subspace is unknown, we showed that, for feedforward neural networks, the change in the first-layer weights during training can be used to approximate this subspace, leading to \autoref{alg:reconstr_subspace}.
Our numerical experiments on synthetic datasets and CIFAR-10 further confirm that this subspace-aware reconstruction approach outperforms standard full-space techniques that do not exploit data structure.

The key motivation for this research was to provide conditions under which data reconstruction may be possible in general and when certain data structures are present. We acknowledge that our reconstruction techniques may raise concerns regarding privacy,
which we hope will lead to further research in developing
mechanisms that protect from such neural network vulnerabilities. In this sense, a take-away message of our work is that we may want to ensure that publicly available models are not operating in the wide network regime where training accuracy is attained through memorization of data points rather than learning generalizing functions.

\begin{figure}
    \centering
    \includegraphics[width=0.44\textwidth]{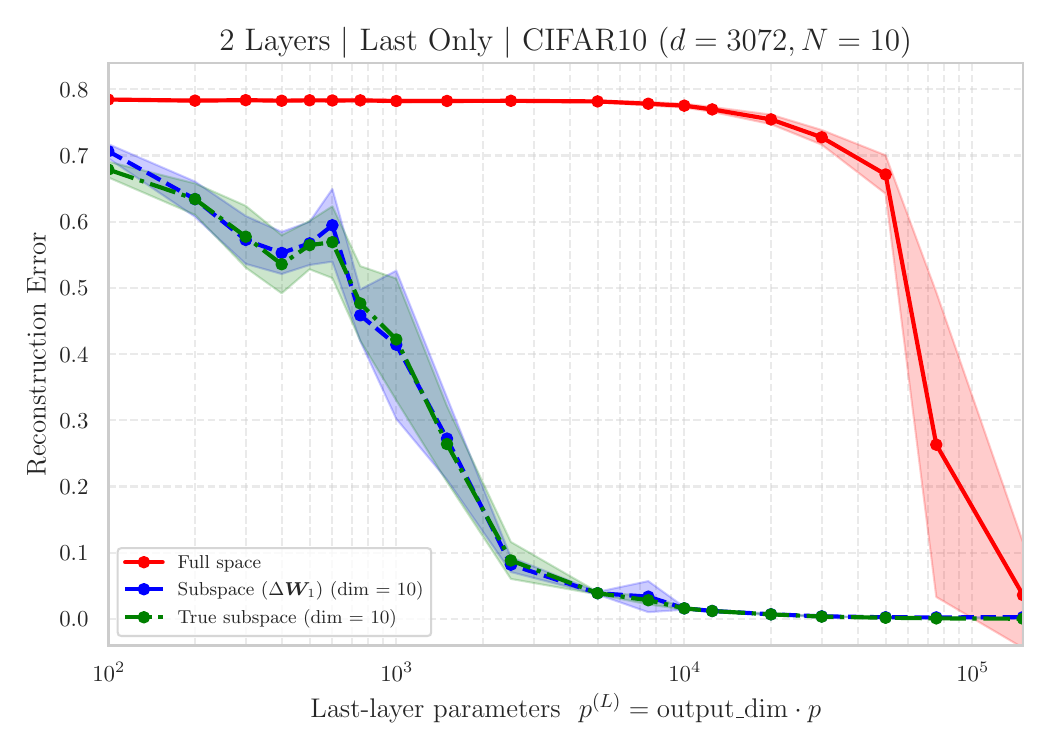}\hfill
    \includegraphics[width=0.44\textwidth]{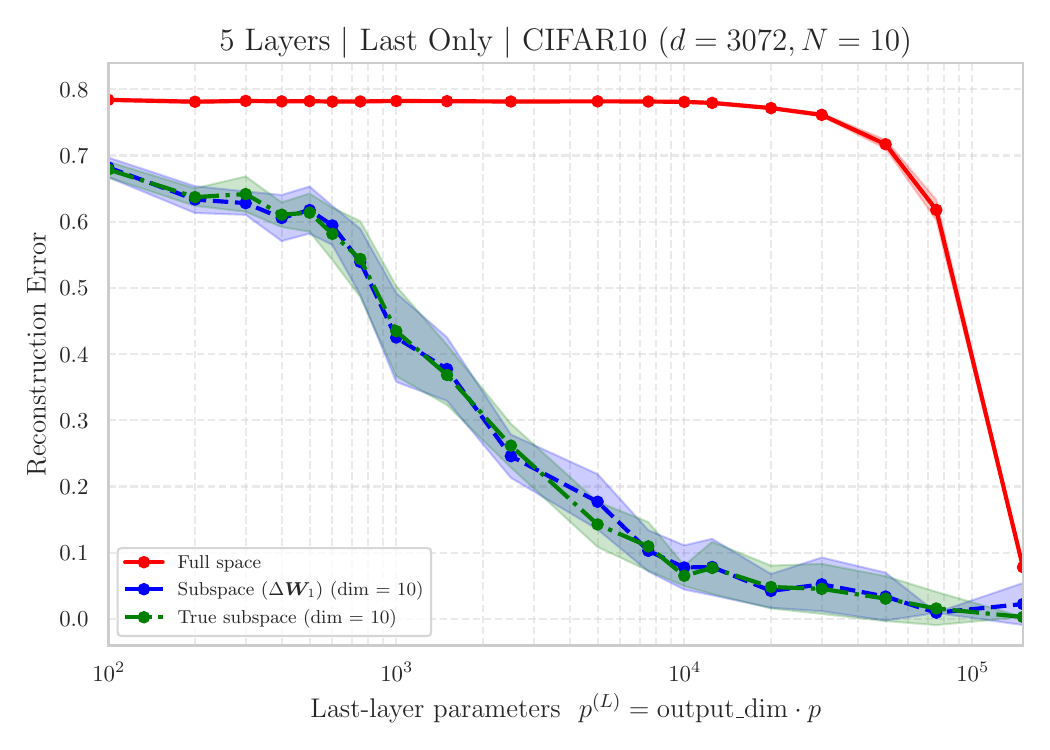}
    \caption{Examining the effect of network depth.
    \autoref{alg:reconstr_full_space} (``Full space'') and \autoref{alg:reconstr_subspace} (``Subspace ($\Delta \mW_1$)'' and ``True subspace'') for CIFAR10 reconstruction using  last-layer parameters only (see Results section). 
    \textbf{Left:} a 2-layer network; \textbf{Right}: a 5-layer network.}
    \label{fig:depth_effect_last_layer}
\end{figure}

\begin{figure}
    \centering
    \includegraphics[width=0.44\textwidth]{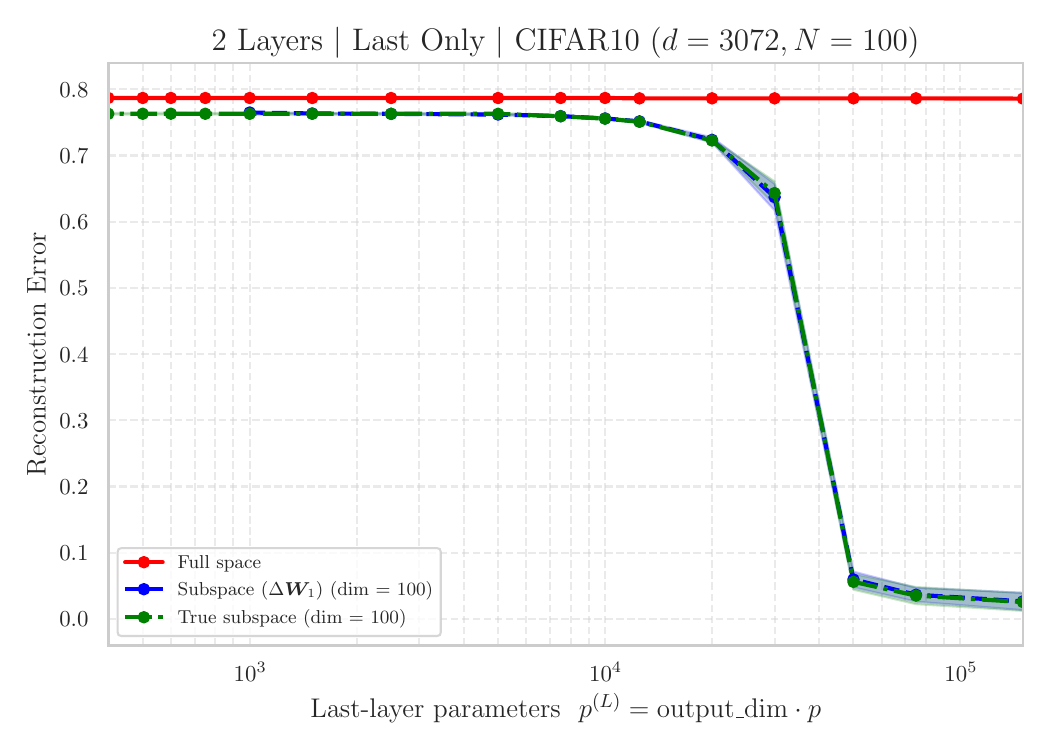}\hfill
    \includegraphics[width=0.44\textwidth]{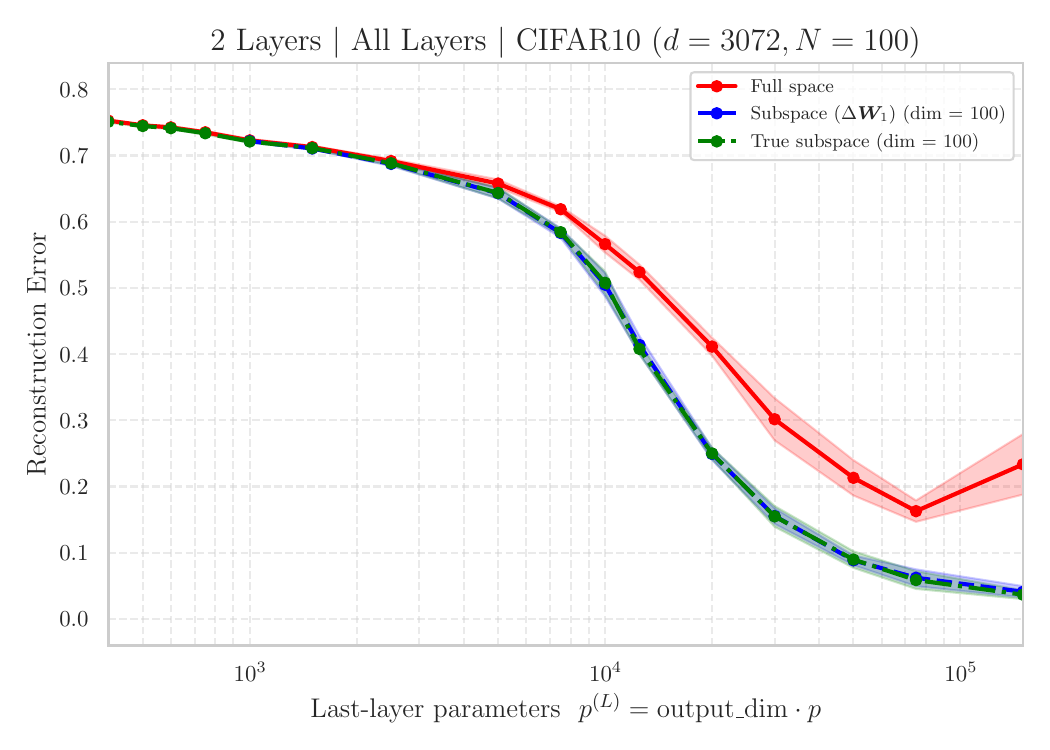}

    \includegraphics[width=0.44\textwidth]{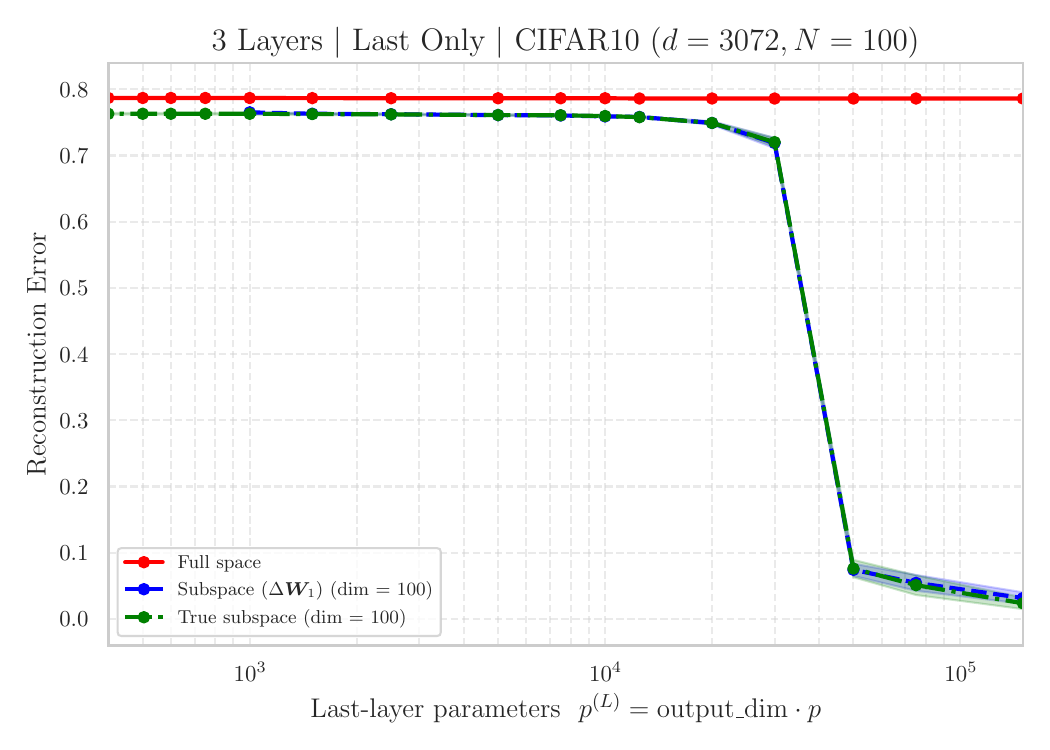}\hfill
    \includegraphics[width=0.44\textwidth]{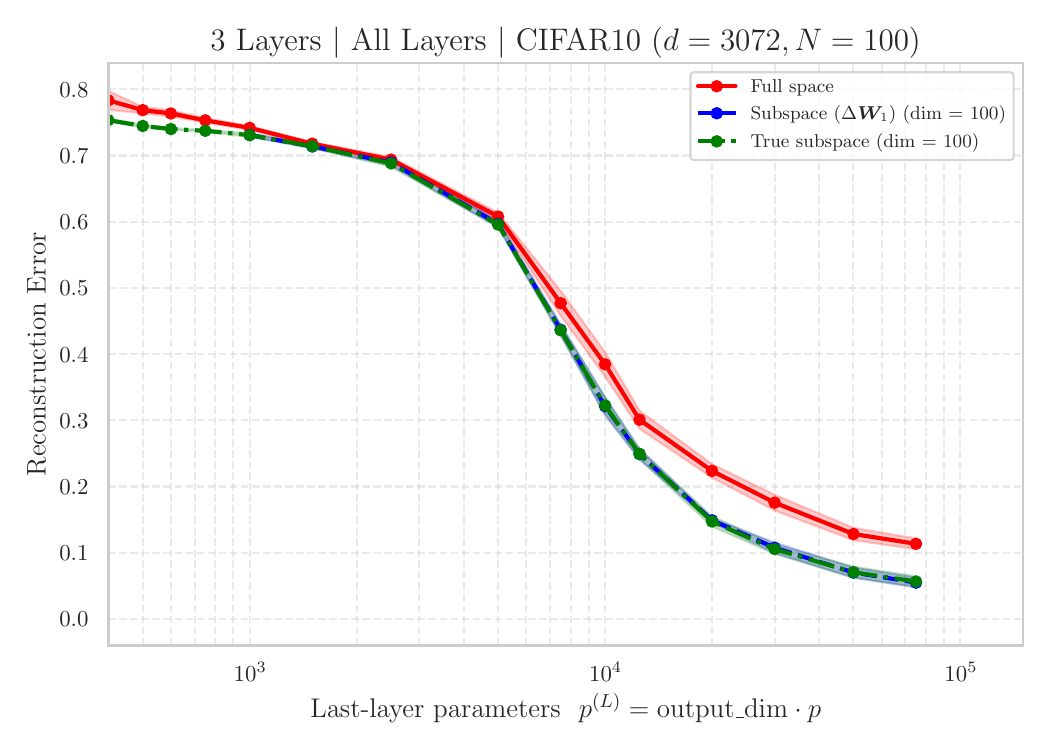}
    \includegraphics[width=0.44\textwidth]{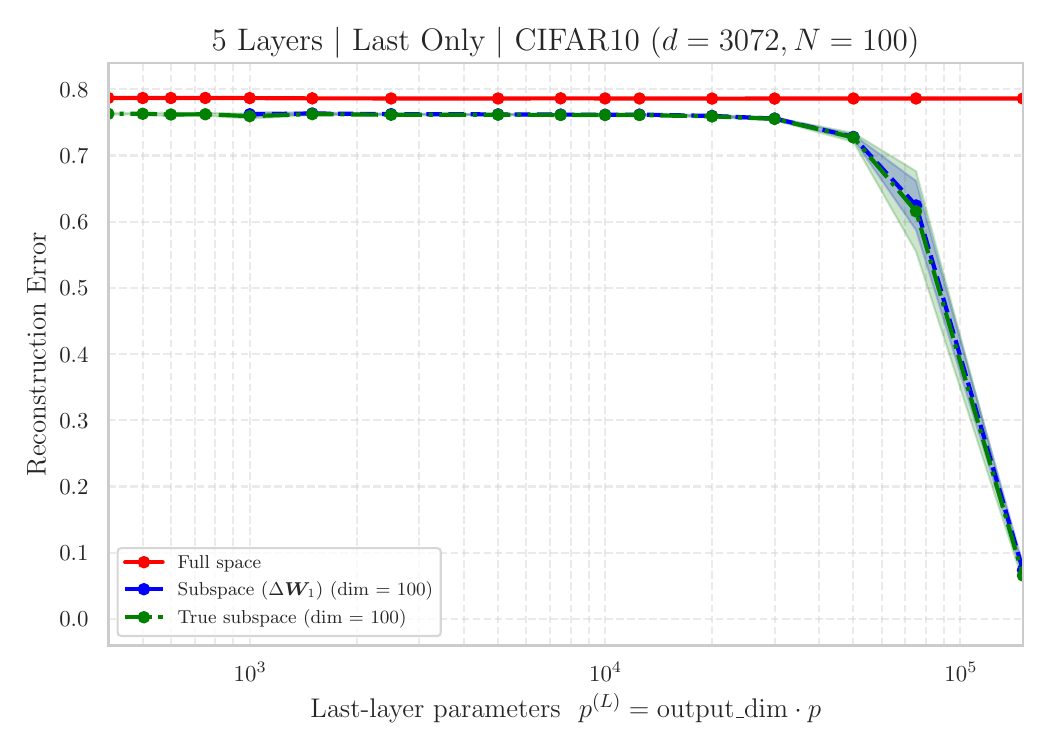}\hfill
    \includegraphics[width=0.44\textwidth]{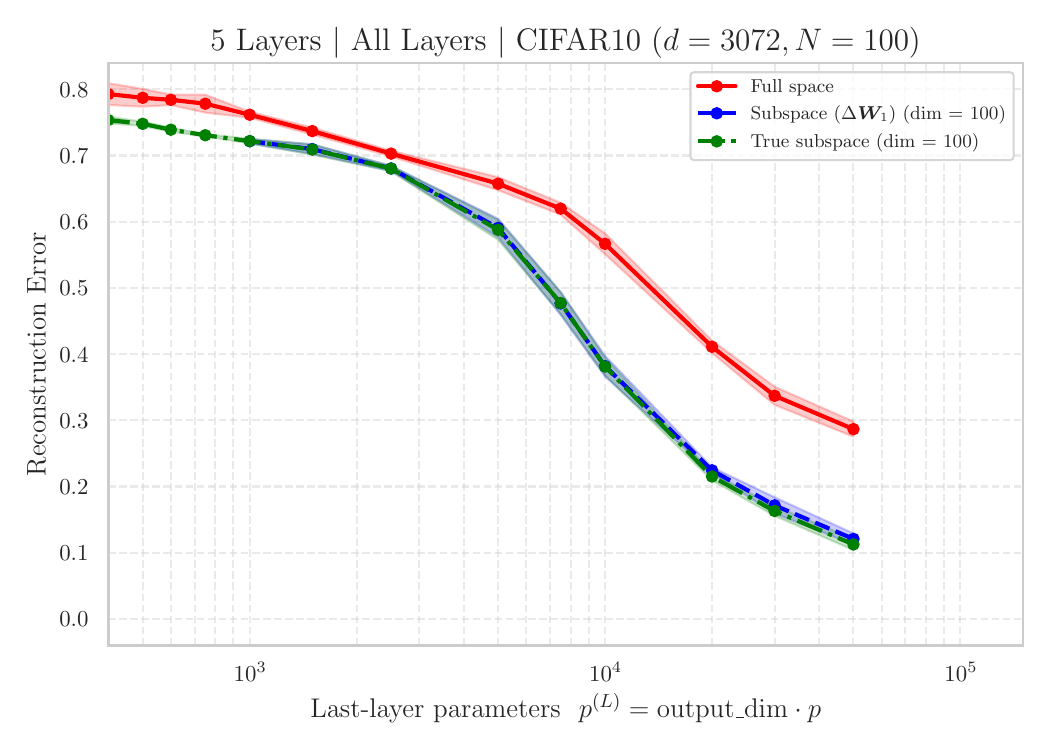}
    
    \caption{Comparing parameter choices with varying network depths (2, 3, 5 layers).   \autoref{alg:reconstr_full_space} (``Full space'') and \autoref{alg:reconstr_subspace} (``Subspace ($\Delta \mW_1$)'' and ``True subspace'') 
    for CIFAR10 with $n=100$ data points, using {\bf last-layer parameters (left column)} vs {\bf all layer parameters (right column)}. \autoref{alg:reconstr_subspace} (blue) with last-layer parameters (only) is performing as well as methods using all parameters in terms of reconstruction error and is superior to them in computational time.}
    \label{fig:first_layer_vs_all_layers}
\end{figure}

\clearpage
\newpage

\bibliography{references}

\clearpage
\newpage

\appendix
\section{Proof of theoretical results}\label{ap:proof_theoretical_results}
We recall Hoeffding's inequality, which will be useful to bound, with high probability, the deviation of the finite-width kernel from the infinite-width one.
\begin{theorem}[Hoeffding's inequality \cite{hoeffding_probability_1963}]\label{thm:Hoeffding_ineq}
    Let $X_1,\ldots,X_p$ be independent random variables such that $a_i\leq X_i\leq b_i$ almost surely. Consider the sum of these random variables, $$S_p=X_1+\ldots+X_p.$$ Then, for all $t>0$, $$\mathbb P[|S_p-\mathbb E[S_p]|\geq t]\leq 2\exp\left(-\frac{2t^2}{\sum_{i=1}^p (b_i-a_i)^2}\right).$$
\end{theorem}

\begin{proof}[(Proof of \autoref{lem:approx_kernel_bound})]
    The general idea of the proof comes from \cite{rahimi_random_2007}, which considers the case of translation-invariant kernels.\\ 
    For a realization of the random variables $\mV$ and $\vb$, define $f(\vx, \vy):=k_p(\vx, \vy)-k_\infty(\vx, \vy)$ for all $\vx,\vy\in \sqrt{d}\mathcal{S}^{d-1}$. Consider an $r$-net cover of this sphere, which by \cite{vershynin_highdimensional_2018} consists of at most $T:=\lfloor(1+2\sqrt{d}/r)^d\rfloor \leq (3\sqrt{d}/r)^d$ balls of radius $r$ ($r\leq\sqrt{d}$) with centers in $\sqrt{d}\mathcal{S}^{d-1}$, and denote these centers by $\{\vc_i\}_{i=1}^T=:\mathcal{C}$. 
    
    Fix $\vx, \vy\in \sqrt{d}\mathcal{S}^{d-1}$. We apply here Hoeffding's inequality (see \autoref{thm:Hoeffding_ineq}) with $$X_i=\frac{1}{p}\phi(\mV_i^\top \vx+b_i)\phi(\mV_i^\top \vy+b_i),$$ where $\mV_i$ denotes the $i$-th row of $\mV$. Clearly, these random variables are independent and satisfy $X_i\in[-\frac{M^2}{p}, \frac{M^2}{p}]$. Moreover, notice that $S_p=k_p(\vx, \vy)$ and $\mathbb E[S_p]=\mathbb E[k_p(\vx, \vy)]=k_\infty(\vx, \vy)$. Hence,  choosing $t=\varepsilon/2$ in \autoref{thm:Hoeffding_ineq}, there holds 

    \begin{align*}
        \mathbb P[|f(\vx, \vy)|\geq\varepsilon/2]&=\mathbb P[|k_p(\vx, \vy)-k_\infty(\vx, \vy)|\geq\varepsilon/2]\\
        &\leq2\exp\left(-\frac{2(\varepsilon/2)^2}{p(2M^2/p)^2}\right)\\ 
        &=2\exp\left(-\frac{p\varepsilon^2}{8M^4}\right).
    \end{align*}
    Applying this inequality to each $\vx_c,\vy_c\in\mathcal{C}$ gives, by the union bound, $$\mathbb P\left[\sup_{\vx_c,\vy_c\in\mathcal{C}}|f(\vx_c,\vy_c)|\geq\varepsilon/2\right]=\mathbb P\left[\bigcup_{\vx_c,\vy_c\in\mathcal{C}}\left\{|f(\vx_c,\vy_c)|\geq\varepsilon/2\right\}\right]\leq2T^2\exp\left(-\frac{p\varepsilon^2}{8M^4}\right).$$
    Moreover, note that $$\nabla_\vx k_p(\vx, \vy)=\frac{1}{p}\sum_{i=1}^p\phi'(\mV_i^\top \vx+b_i)\phi(\mV_i^\top \vy+b_i)\mV_i$$ and $$\nabla_\vx k_\infty(\vx, \vy)=\mathbb E[\phi'(\vv^\top \vx+\vb)\phi(\vv^\top \vy+\vb)\vv].$$ Hence, \begin{align*}
        \|\nabla_\vx f(\vx, \vy)\|&\leq\|\nabla_\vx k_p(\vx, \vy)\|+\|\nabla_\vx k_\infty(\vx, \vy)\|\\
        &\leq \frac{1}{p}\sum_{i=1}^p|\phi'(\mV_i^\top \vx+\vb_i)\phi(\mV_i^\top \vy+\vb_i)|\|\mV_i\|+\mathbb E[|\phi'(\vv^\top \vx+\vb)\phi(\vv^\top \vy+\vb)|\|\vv\|]\\
        &\leq \frac{LM}{p}\sum_{i=1}^p\|\mV_i\|+LM\mathbb E[\|\vv\|].
    \end{align*} Let $L_\vx:=\sup_{\vx,\vy\in \sqrt{d}\mathcal{S}^{d-1}}\|\nabla_\vx f(\vx,\vy)\|$. It follows that \begin{align*}
        \mathbb E[L_\vx]&\leq \frac{LM}{p}\sum_{i=1}^p\mathbb E[\|\mV_i\|]+LM\mathbb E[\|\vv\|]\\
        &=2LM\mathbb E[\|\vv\|]\\
        &\leq2LM\sqrt{\mathbb E[\|\vv\|^2]}\\
        &=2LM.
    \end{align*} Therefore, by Markov's inequality, $$\mathbb P[L_\vx\geq \varepsilon/(4r)]\leq 4r\mathbb E[L_\vx]/\varepsilon\leq (8rLM)/\varepsilon.$$ Additionally, notice that $L_\vy:=\sup_{\vx,\vy\in \sqrt{d}\mathcal{S}^{d-1}}\|\nabla_\vy f(\vx,\vy)\|=L_\vx$ by symmetry of the kernels. Finally, if \begin{equation}\label{eq:kernel_proof_condition}
        \sup_{\vx_c,\vy_c\in\mathcal{C}}|f(\vx_c,\vy_c)|<\varepsilon/2\quad\text{and}\quad L_\vx< \varepsilon/(4r)
    \end{equation} then, for all $\vx,\vy\in \sqrt{d}\mathcal{S}^{d-1}$ and $\vx',\vy'\in\mathcal{C}$ such that $\|\vx-\vx'\|\leq r$ and $\|\vy-\vy'\|\leq r$ (which are possible to find since $\mathcal{C}$ is an $r$-net cover of $\sqrt{d}\mathcal{S}^{d-1}$), \begin{align*}
        |f(\vx,\vy)|&\leq |f(\vx,\vy)-f(\vx',\vy')|+|f(\vx',\vy')|\\
        &\leq |f(\vx,\vy)-f(\vx,\vy')| + |f(\vx,\vy')-f(\vx',\vy')|+|f(\vx',\vy')|\\
        &\leq L_\vy\|\vy-\vy'\|+L_\vx\|\vx-\vx'\|+\sup_{\vx_c,\vy_c\in\mathcal{C}}|f(\vx_c,\vy_c)|\\
        &< 2L_\vx r+\varepsilon/2\\
        &< \varepsilon
    \end{align*} and hence $\sup_{\vx,\vy\in \sqrt{d}\mathcal{S}^{d-1}}|f(\vx,\vy)|\leq \varepsilon$. Condition (\ref{eq:kernel_proof_condition}) holds with probability \begin{align*}
        \mathbb P\left[\sup_{\vx_c,\vy_c\in\mathcal{C}}|f(\vx_c,\vy_c)|<\varepsilon/2\quad\text{and}\quad L_\vx< \varepsilon/(4r)\right]&=1-\mathbb P\left[\left\{\sup_{\vx_c,\vy_c\in\mathcal{C}}|f(\vx_c,\vy_c)|\geq\varepsilon/2\right\}\cup \left\{L_\vx\geq\varepsilon/(4r)\right\}\right]\\
        &\geq 1-\mathbb P\left[\sup_{\vx_c,\vy_c\in\mathcal{C}}|f(\vx_c,\vy_c)|\geq\varepsilon/2\right]-\mathbb P[L_\vx\geq\varepsilon/(4r)]\\
        &\geq 1-2T^2\exp\left(-\frac{p\varepsilon^2}{8M^4}\right)-(8rLM)/\varepsilon,
    \end{align*} where we have used the union bound in the penultimate inequality. Therefore, $$\mathbb P\left[\sup_{\vx,\vx'\in \sqrt{d}\mathcal{S}^{d-1}}|k_p(\vx, \vx')-k_\infty(\vx, \vx')|\leq\varepsilon\right]\geq1-2\left(\frac{3\sqrt{d}}{r}\right)^{2d}\exp\left(-\frac{p\varepsilon^2}{8M^4}\right)-\frac{8rLM}{\varepsilon}.$$ As $0<r\leq\sqrt{d}$ can vary freely, we maximize the right-hand side of the above expression with respect to $r$, by first ignoring the upper bound on $r$. In this regard, denote the right-hand sight of the above inequality by $g(r)$. We have $g(r)=1-Ar^{-2d}-Br$, with $$A:=2\cdot(3\sqrt{d})^{2d}\exp\left(-\frac{p\varepsilon^2}{8M^4}\right)\quad\text{and}\quad B:=\frac{8LM}{\varepsilon},$$ which are both positive. Moreover, $$g'(r)=2d\frac{A}{r^{2d+1}}-B\quad\text{and}\quad g''(r)=-2d(2d+1)\frac{A}{r^{2d+2}}.$$ Clearly, $g''(r)<0$ for all $r>0$, $\lim_{r\to0^+}g(r)=\lim_{r\to+\infty}g(r)=-\infty$ and $g$ admits a unique critical point $r^*$. Hence, $r^*$ is a global maximizer of $g$ on $(0,+\infty)$ and is given by $$r^*=\left(2d\frac{A}{B}\right)^{\frac{1}{2d+1}}=\left(\frac{d(3\sqrt{d})^{2d}\varepsilon}{2LM}\right)^{\frac{1}{2d+1}}\exp\left(-\frac{p\varepsilon^2}{8M^4(2d+1)}\right).$$ To ensure $r^*\leq\sqrt{d}$, $p$ should then satisfy $$\left(\frac{d(3\sqrt{d})^{2d}\varepsilon}{2LM}\right)^{\frac{1}{2d+1}}\exp\left(-\frac{p\varepsilon^2}{8M^4(2d+1)}\right)\leq \sqrt{d}$$ $$\Leftrightarrow\exp\left(-\frac{p\varepsilon^2}{8M^4}\right)\leq \frac{2LM\sqrt{d}}{d3^{2d}\varepsilon}$$ \begin{equation}\label{eq:p_condition_for_r}
        \Leftrightarrow p\geq\frac{8M^4}{\varepsilon^2}\ln\left(\frac{d3^{2d}\varepsilon}{2LM\sqrt{d}}\right).
    \end{equation} Furthermore, \begin{align*}
        g(r^*)&=1-\left(\frac{A}{{r^*}^{2d+1}}+B\right)r^*\\
        &=1-\left(\frac{B}{2d}+B\right)\left(2d\frac{A}{B}\right)^{\frac{1}{2d+1}}\\
        &=1-\frac{1+2d}{2d}B(2d)^{\frac{1}{2d+1}}\left(\frac{1}{B}\right)^{\frac{1}{2d+1}}A^{\frac{1}{2d+1}}\\
        &=1-(1+2d)\left(\frac{B}{2d}\right)^{\frac{2d}{2d+1}}A^{\frac{1}{2d+1}}\\
        &=1-(1+2d)\left(\frac{4LM}{d\varepsilon}\right)^{\frac{2d}{2d+1}}(2\cdot(3\sqrt{d})^{2d})^{\frac{1}{2d+1}}\exp\left(-\frac{p\varepsilon^2}{8M^4(2d+1)}\right)\\
        &=1-(1+2d)\left(\frac{2LM}{d\varepsilon}\right)^{\frac{2d}{2d+1}}2^{\frac{2d}{2d+1}}2^{\frac{1}{2d+1}}(3\sqrt{d})^{\frac{2d}{2d+1}}\exp\left(-\frac{p\varepsilon^2}{8M^4(2d+1)}\right)\\
        &=1-(2+4d)\left(\frac{6LM\sqrt{d}}{d\varepsilon}\right)^{\frac{2d}{2d+1}}\exp\left(-\frac{p\varepsilon^2}{8M^4(2d+1)}\right).
    \end{align*}
    For $g(r^*)$ to be larger than $1-\delta$, where $\delta>0$, $p$ should then satisfy $$1-(2+4d)\left(\frac{6LM\sqrt{d}}{d\varepsilon}\right)^{\frac{2d}{2d+1}}\exp\left(-\frac{p\varepsilon^2}{8M^4(2d+1)}\right)\geq 1-\delta$$ $$\Leftrightarrow\frac{\delta}{2+4d}\left(\frac{d\varepsilon}{6LM\sqrt{d}}\right)^{\frac{2d}{2d+1}}\geq\exp\left(-\frac{p\varepsilon^2}{8M^4(2d+1)}\right)$$ $$\Leftrightarrow\ln\left(\frac{\delta}{2+4d}\right)+\frac{2d}{2d+1}\ln\left(\frac{d\varepsilon}{6LM\sqrt{d}}\right)\geq-\frac{p\varepsilon^2}{8M^4(2d+1)}$$ $$\Leftrightarrow p\geq \frac{8M^4(2d+1)}{\varepsilon^2}\left(\ln\left(\frac{2+4d}{\delta}\right)+\frac{2d}{2d+1}\ln\left(\frac{6LM\sqrt{d}}{d\varepsilon}\right)\right).$$ For $r^*$ to be less than $\sqrt{d}$, it is therefore enough to choose $p$ as above and a $\delta$ such that the right-hand side of the above inequality is larger than the one of (\ref{eq:p_condition_for_r}), i.e., $$\frac{8M^4(2d+1)}{\varepsilon^2}\left(\ln\left(\frac{2+4d}{\delta}\right)+\frac{2d}{2d+1}\ln\left(\frac{6LM\sqrt{d}}{d\varepsilon}\right)\right)\geq \frac{8M^4}{\varepsilon^2}\ln\left(\frac{d3^{2d}\varepsilon}{2LM\sqrt{d}}\right)$$ $$\Leftrightarrow(2d+1)\ln\left(\frac{2+4d}{\delta}\right)\geq\ln\left(\frac{d3^{2d}\varepsilon}{2LM\sqrt{d}}\right)-2d\ln\left(\frac{6LM\sqrt{d}}{d\varepsilon}\right)$$ $$\Leftrightarrow(2d+1)\ln\left(\frac{2+4d}{\delta}\right)\geq\ln\left(\frac{d3^{2d}\varepsilon}{2LM\sqrt{d}}\left(\frac{d\varepsilon}{6LM\sqrt{d}}\right)^{2d}\right)$$ $$\Leftrightarrow(2d+1)\ln\left(\frac{2+4d}{\delta}\right)\geq\ln\left(\left(\frac{d\varepsilon}{2LM\sqrt{d}}\right)^{2d+1}\right)$$ $$\Leftrightarrow\ln\left(\frac{2+4d}{\delta}\right)\geq\ln\left(\frac{d\varepsilon}{2LM\sqrt
    d}\right)$$ $$\Leftrightarrow\delta\leq \frac{4(1+2d)LM}{\sqrt{d}\varepsilon},$$ which concludes the proof.
\end{proof}

\begin{proof}[(Proof of \autoref{thm:finite_recon_proof})]
If $p$ satisfies condition (\ref{eq:p_condition}) in \autoref{lem:approx_kernel_bound}, then, with probability $1-\delta$, using (\ref{eq:theta_opt_value}), we have
\begin{align*}
    \mathcal{L}_{\text{recon}}(\hat \mX,\hat\valpha)&=\|\vtheta^*-\sum_{i=1}^n\hat\alpha_i\varphi(\hat \vx_i)\|_2^2\\
    &=\left\|\sum_{i=1}^n\alpha_i\varphi(\vx_i)-\sum_{i=1}^n\hat\alpha_i\varphi(\hat \vx_i)\right\|_2^2\\
    &=\left\|\sum_{i=1}^n\alpha_i k_p(\vx_i,\cdot)-\sum_{i=1}^n\hat\alpha_i k_p(\hat\vx_i, \cdot)\right\|_{\mathcal{H}_p}^2\\
    &=\left\|\sum_{i=1}^n\alpha_i k_\infty(\vx_i,\cdot)-\sum_{i=1}^n\hat\alpha_i k_\infty(\hat\vx_i, \cdot)\right\|_{\mathcal{H}_\infty}^2+\left\|\sum_{i=1}^n\alpha_i k_p(\vx_i,\cdot)-\sum_{i=1}^n\hat\alpha_i k_p(\hat\vx_i, \cdot)\right\|_{\mathcal{H}_p}^2\\ &\qquad - \left\|\sum_{i=1}^n\alpha_i k_\infty(\vx_i,\cdot)-\sum_{i=1}^n\hat\alpha_i k_\infty(\hat\vx_i, \cdot)\right\|_{\mathcal{H}_\infty}^2\\
    &=\left\|\sum_{i=1}^n\alpha_i k_\infty(\vx_i,\cdot)-\sum_{i=1}^n\hat\alpha_i k_\infty(\hat\vx_i, \cdot)\right\|_{\mathcal{H}_\infty}^2+\sum_{i,j}\alpha_i\alpha_j(k_p(\vx_i,\vx_j)-k_\infty(\vx_i,\vx_j))\\
    &\qquad-2\sum_{i,j}\alpha_i\hat\alpha_j(k_p(\vx_i,\hat \vx_j)-k_\infty(\vx_i,\hat \vx_j))+\sum_{i,j}\hat\alpha_i\hat\alpha_j(k_p(\hat \vx_i,\hat \vx_j)-k_\infty(\hat \vx_i,\hat \vx_j))\\
    &\geq\left\|\sum_{i=1}^n\alpha_i k_\infty(\vx_i,\cdot)-\sum_{i=1}^n\hat\alpha_i k_\infty(\hat\vx_i, \cdot)\right\|_{\mathcal{H}_\infty}^2-\varepsilon\left(\sum_{i,j}|\alpha_i||\alpha_j| +2\sum_{i,j}|\alpha_i||\hat\alpha_j|+\sum_{i,j}|\hat\alpha_i||\hat\alpha_j|\right)\\
    &=\left\|\sum_{i=1}^n\alpha_i k_\infty(\vx_i,\cdot)-\sum_{i=1}^n\hat\alpha_i k_\infty(\hat\vx_i, \cdot)\right\|_{\mathcal{H}_\infty}^2-\varepsilon(\|\valpha\|_1+\|\hat \valpha\|_1)^2\\
    &=\left\|\int_{\sqrt{d}\mathcal{S}^{d-1}}k_\infty(\vx, \cdot)\, d P-\int_{\sqrt{d}\mathcal{S}^{d-1}}k_\infty(\hat\vx, \cdot)\, d \hat P\right\|_{\mathcal{H}_\infty}^2-\varepsilon(\|\valpha\|_1+\|\hat \valpha\|_1)^2,
\end{align*} where $$P:=\sum_{i=1}^n\alpha_i\delta(\vx_i)\quad\text{and}\quad\hat P:=\sum_{i=1}^n\hat\alpha_i\delta(\hat\vx_i).$$
The term with the RKHS norm is the maximum mean discrepancy (MMD) squared between the signed measures $P$ and $\hat P$ \cite{gretton2012kernel}, which we denote by $\mathrm{MMD}_{k_\infty}^2(P,\hat P)$. It follows that, as $\mathcal{L}_{\text{recon}}(\hat \mX,\hat\valpha)=0$, $$\mathrm{MMD}_{k_\infty}(P,\hat P)\leq \sqrt{\varepsilon}(\|\valpha\|_1+\|\hat \valpha\|_1).$$
By \cite{vayer_controlling_2023}, the MMD is equivalent to $$\mathrm{MMD}_{k_\infty}(P, \hat P)=\sup_{\substack{g\in\mathcal{H}_\infty\\ \|g\|_{\mathcal{H}_\infty}\leq1}}\left|\int_{\sqrt{d}\mathcal{S}^{d-1}}g\, d P-\int_{\sqrt{d}\mathcal{S}^{d-1}} g\, d\hat P\right|.$$ Therefore, for any $g\in\mathcal{H}_\infty$, $$\left|\sum_{i=1}^n\alpha_i g(\vx_i)-\sum_{i=1}^n\hat\alpha_i g(\hat \vx_i)\right|\leq\sqrt{\varepsilon}\|g\|_{\mathcal{H}_\infty}(\|\valpha\|_1+\|\hat\valpha\|_1).$$
Moreover, under the conditions of the statement on the activation function, by \cite{sun_approximation_2019}, $k_\infty$ is universal on the sphere $\sqrt{d}\mathcal{S}^{d-1}$ . It follows that for any continuous function $f:\sqrt{d}\mathcal{S}^{d-1}\to\R$, there exists a function $g\in\mathcal{H}_\infty$ such that $\|f-g\|_\infty\leq\tilde{\varepsilon}$, where $\tilde{\varepsilon}>0$. Hence, \begin{align*}
    \left|\sum_{i=1}^n\alpha_i f(\vx_i)-\sum_{i=1}^n\hat\alpha_i f(\hat \vx_i)\right|&\leq\left|\sum_{i=1}^n\alpha_i f(\vx_i)-\sum_{i=1}^n\alpha_i g(\vx_i)\right|+\left|\sum_{i=1}^n\alpha_i g(\vx_i)-\sum_{i=1}^n\hat\alpha_i g(\hat \vx_i)\right|\\ &\qquad+\left|\sum_{i=1}^n\hat\alpha_i g(\hat \vx_i)-\sum_{i=1}^n\hat\alpha_i f(\hat \vx_i)\right|\\
    &\leq \sqrt{\varepsilon}\|g\|_{\mathcal{H}_\infty}(\|\valpha\|_1+\|\hat \valpha\|_1)+\tilde{\varepsilon}(\|\valpha\|_1+\|\hat\valpha\|_1)\\
    &=(\sqrt{\varepsilon}\|g\|_{\mathcal{H}_\infty}+\tilde{\varepsilon})(\|\valpha\|_1+\|\hat\valpha\|_1).
\end{align*} For each $j\in[n]$, let the continuous function $f_j:\sqrt{d}\mathcal{S}^{d-1}\to\R$ be defined as
\begin{equation} \label{eq:cont_fct}
    f_j(\vx):=\begin{cases}
    \exp\left(\dfrac{\|\vx-\vx_j\|_2^2}{\|\vx-\vx_j\|_2^2-\Delta^2}\right) &\text{if $\|\vx-\vx_j\|_2<\Delta$,}\\
    0&\text{otherwise,}
\end{cases}
\end{equation}
and $g_j\in\mathcal{H}_\infty$ be such that $\|f_j-g_j\|_\infty\leq\tilde{\varepsilon}$, with $$\tilde{\varepsilon}:=\frac{c}{4(\|\valpha\|_1+\|\hat\valpha\|_1)}.$$ Furthermore, let $C:=\displaystyle\max_{j\in[n]}\|g_j\|_{\mathcal{H}_\infty}^2$ and $$\varepsilon=\frac{c^2}{4C(\|\valpha\|_1+\|\hat\valpha\|_1)^2},$$ which, by \autoref{lem:approx_kernel_bound}, can be ensured if $\delta$ and $p$ are as in the statement. 

Choose $f= f_j$ for some $j \in [n]$ in (\ref{eq:cont_fct}), and assume by contradiction that $\|\vx_j-\hat\vx_i\|\geq\Delta$ for all $i\in[n]$. It follows that 
\begin{align*}
    \left|\sum_{i=1}^n\alpha_i f_j(\vx_i)-\sum_{i=1}^n\hat\alpha_i f_j(\hat \vx_i)\right|&=|\alpha_j|\\
    &\leq(\sqrt{\varepsilon}\|g_j\|_{\mathcal{H}_\infty}+\tilde{\varepsilon})(\|\valpha\|_1+\|\hat\valpha\|_1)\\
    &\leq\frac{c}{2}+\frac{c}{4}\\
    &<c,
\end{align*} which is a contradiction since $|\alpha_i|\geq c$ for all $i\in[n]$.
Hence, for each training sample $\vx_j$, there exists a reconstructed sample $\hat\vx_i$ such that $\|\vx_j-\hat\vx_i\|_2<\Delta$, which concludes the proof.
\end{proof}

\section{Additional information on numerical experiments}
\label{sec:further_numerics_details}

Here we provide further details of the numerical experiments. We will release the code used to conduct the experiments upon publication.

Our experiments were conducted on a range of servers. Smaller training and reconstruction tasks (typically $p \leq 1000$) were conducted on a 64 core CPU system. Larger experiments were conducted on systems with either 4x RTX 3080ti GPUs or 4x Nvidia H100 GPUs. The latter being required due to the memory requirements of the all-layer reconstruction method for the widest networks.

\paragraph{Computing the loss function}

We use the same observation regarding the symmetry and idempotence of $\mP_{\mG}^\perp$ from \cite{iurada_law_2025}, that

\begin{align*}
 \|\mP_{\mG}^\perp (\Delta \vtheta)\|_2^2
    &= (\Delta\vtheta)^{\top}(\Delta\vtheta) - (\Delta\vtheta)^{\top}\mG^{\top}(\mG\mG^{\top})^{-1}\mG (\Delta\vtheta) \\
    &= (\Delta\vtheta)^{\top}(\Delta\vtheta) - (\Delta\vtheta)^{\top}\mG^{\top}\valpha,
\end{align*}

where $\valpha$ solves
\begin{equation}\label{eq:linear_system_for_alpha}
    (\mG\mG^{\top})\valpha = \mG(\Delta\vtheta).
\end{equation}
For deep networks, including those with multiclass outputs, computing the matrix $\mG$ is likely to be infeasible.
In these cases, we may make use of Jacobian-vector and vector-Jacobian products, both of which are implemented in PyTorch. Specifically, we can compute $\mG(\Delta\vtheta)$ through a Jacobian-vector product. To solve (\ref{eq:linear_system_for_alpha}), we can make use of the conjugate gradient (CG) method. The CG method requires the ability to compute matrix-vector products with $\mG\mG^{\top}$. For a given vector $\vv$, we may compute $\mG\mG^{\top}\vv$ through two Jacobian-vector products as
$(\mG\mG^{\top})\vv =  \mG\vu$ for $\vu = (\vv^{\top}\mG)^{\top}$.

For reconstruction methods using the final layer only, it is typically feasible to explicitly form the matrix $\mG\mG^{\top}$ and avoid using the CG method, which leads to cost savings and greater scalability of the method.
The gradient of the loss with respect to the reconstruction data $\hat{\mX}$ can be computed via automatic differentiation.

There is a clear trade-off between the memory requirements. Methods that explicitly compute $\mG$ have a greater memory requirement than methods that use CG. However, these methods run significantly faster.

\section{Additional experiments}
\label{sec:additional_plots}
Here we include the results of additional numerical experiments.

In \autoref{fig:first_layer_vs_all_layers_appendix}, we compare the reconstruction schemes using the last layer only with the reconstruction using both the network layers, for 2-layer networks, for the CIFAR-10 ($n=10$) and the synthetic dataset. In both cases, we see that the subspace methods perform similarly between the ``last layer'' and ``all layers'' reconstruction schemes. We see that for CIFAR-10, the full-space reconstruction algorithm performs comparatively to the subspace reconstruction methods for the ``all layers'' reconstruction, however, for the synthetic data this is not the case.

In \autoref{fig:n=100_reconstruction_2per_class}, we include the results of the reconstruction algorithm, plotting the best reconstructions after training a $p=1000$ width network on $n=100$ CIFAR-10 images. There, we plot the first $2$ images from each class. 

\begin{figure}
    \centering
    \includegraphics[width=0.48\textwidth]{Figures/updated_final/reconstruction_accuracy_final_layer_only_single_recon_final_layer_only_synthetic_n=100_d=60_data_r=30_depths=1_depth=1.pdf}
    \hfill
    \includegraphics[width=0.48\textwidth]{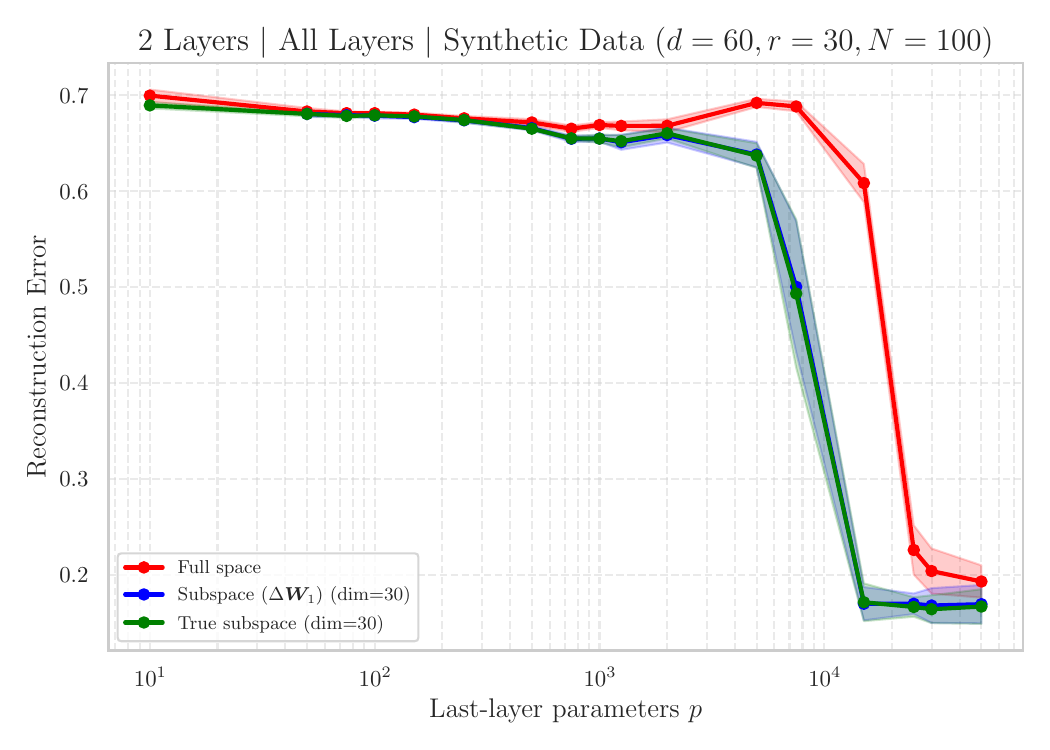}

    \includegraphics[width=0.48\textwidth]{Figures/updated_final/reconstruction_accuracy_last_vs_all_single_recon_last_vs_all_n=10_depths=1-4_Unit_norm_vFast_all_seeds_last_depth=1.pdf}\hfill
    \includegraphics[width=0.48\textwidth]{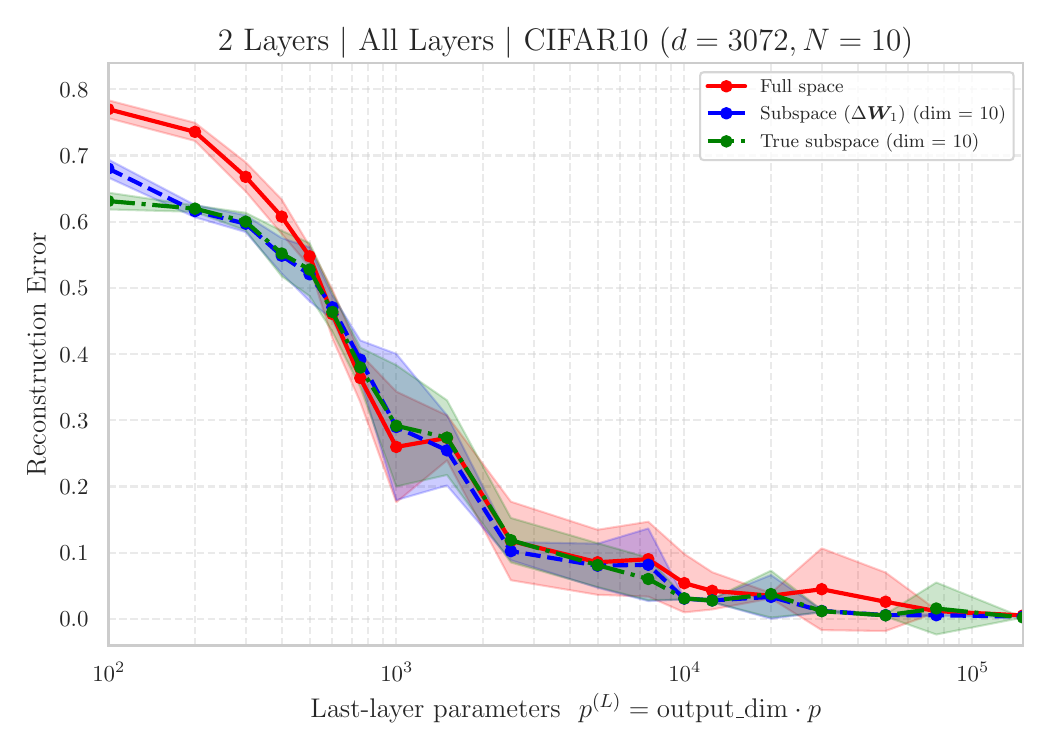}
    
    \caption{Comparing the reconstruction performance of the last layer reconstruction scheme with the all layer scheme. \textbf{Bottom row}: Synthetic data; \textbf{Top row}: CIFAR-10 data}
    \label{fig:first_layer_vs_all_layers_appendix}
\end{figure}

\begin{figure}
    \centering
    \includegraphics[width=0.65\textwidth]{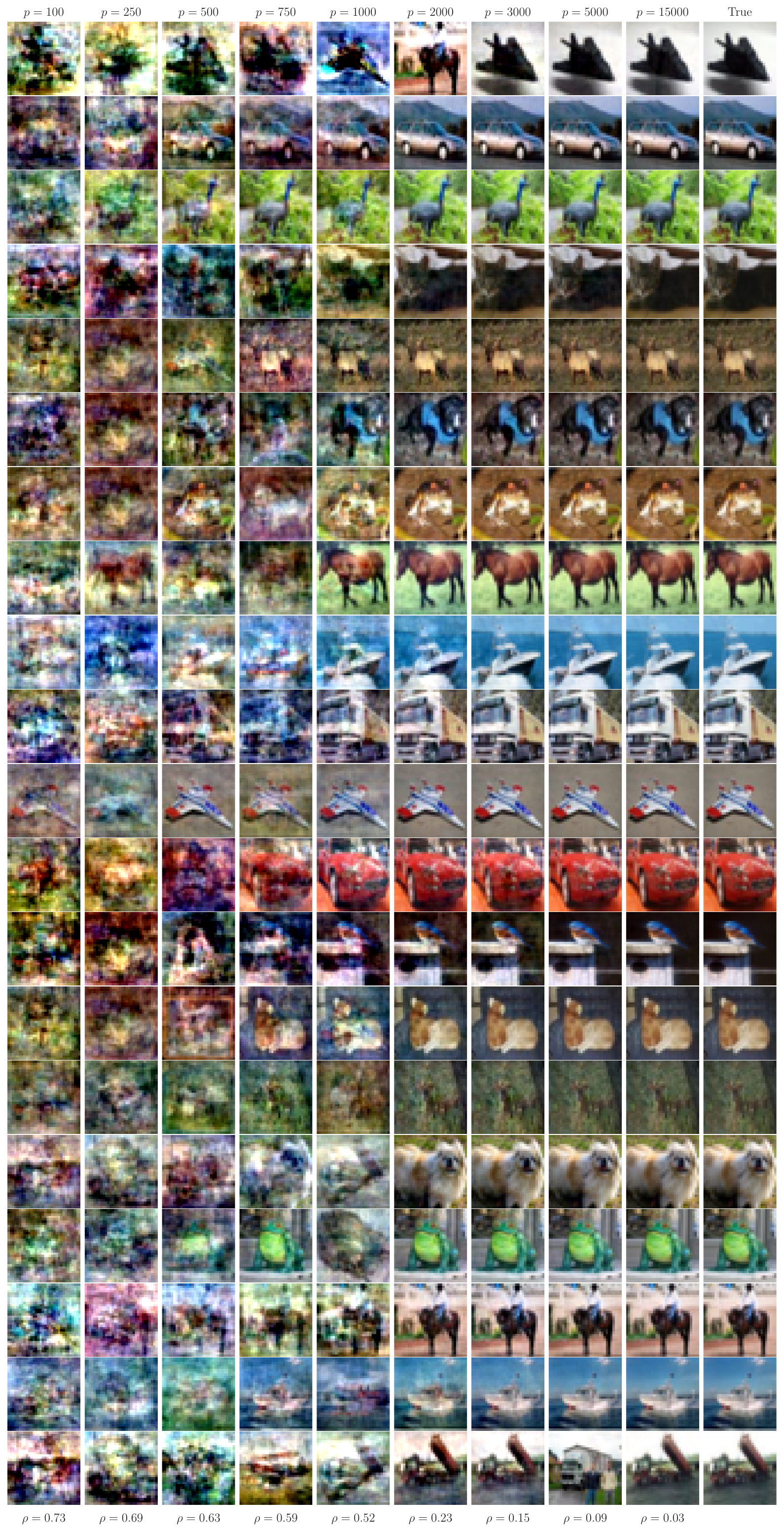}
    
    \caption{How the quality of reconstructed images varies along with the reconstruction error $\rho$ averaged over the entire dataset ($n=100$). We plot the closest reconstructed image to the true image for each width $p$. We plot the first two images from each of the 10 classes.}
    \label{fig:n=100_reconstruction_2per_class}
\end{figure}

\end{document}